%% file: neurips_2025.tex
\documentclass{article}

\usepackage[final]{neurips_2025}

\usepackage[utf8]{inputenc} 
\usepackage[T1]{fontenc}    
\usepackage{hyperref}       
\usepackage{url}            
\usepackage{booktabs}       
\usepackage{amsfonts}       
\usepackage{nicefrac}       
\usepackage{microtype}      
\usepackage{xcolor}         
\input{header}
\title{Scale-invariant attention}

\author{%
  Ben Anson\\
  School of Mathematics\\
  University of Bristol\\
  \texttt{ben.anson@bristol.ac.uk}
  \And
  Xi Wang\\
  Department of Computer Science\\
  Johns Hopkins University\\
  \texttt{xidulu@gmail.com}
  \And
  Laurence Aitchison\\
  School of Computer Science\\
  University of Bristol\\
  \texttt{laurence.aitchison@bristol.ac.uk}
}

\begin{document}
\maketitle
\input{body}
\end{document}

%% file: header.tex
\usepackage{amsfonts,amsmath,amssymb,amsthm,mathtools}
\usepackage{natbib}
\usepackage{booktabs}
\usepackage{bm}
\usepackage{placeins}

\usepackage{enumitem}
\usepackage{subcaption}

\usepackage{thmtools}

\newtheorem{definition}{Definition}[section]

\newcommand{\bracket}[3]{\left#1 #3 \right#2}

\renewcommand{\b}{\bracket{(}{)}}
\newcommand{\sqb}{\bracket{[}{]}}

\newcommand{\N}{\mathcal{N}\b}

\newcommand{\bareE}{\operatorname{E}}

\newcommand{\E}[1][]{\bareE_{#1}\sqb}

\hypersetup{%
    pdfborder = {0 0 0},
    colorlinks,
    citecolor=blue,
    linkcolor=blue,
}

%% file: body.tex
\newcommand{\prope}[0]{$p$-RoPE}

\begin{abstract}
One persistent challenge in LLM research is the development of attention mechanisms that are able to generalise from training on shorter contexts to inference on longer contexts.
We propose two conditions that we expect all effective long-context attention mechanisms to have: scale-invariant total attention, and scale-invariant attention sparsity.
Under a Gaussian assumption, we show that a simple position-dependent transformation of the attention logits is sufficient for these conditions to hold.
Experimentally we find that the resulting scale-invariant attention scheme gives considerable benefits in terms of validation loss when zero-shot generalising from training on short contexts to validation on longer contexts, and is effective at long-context retrieval.
\end{abstract}
\section{Introduction}
One key challenge in modern LLMs is scaling up context length at inference time, while maintaining model performance. We approach this question of length generalisation by considering scale invariance.
In particular, we are inspired by the ``scale-invariant'' statistics of natural images \citep{van1996modelling}.
Scale invariance, for images, is actually highly intuitive, and means that there is structure at all spatial scales.
For example, in an image there might be big features that are 100--1000 pixels across, and some small features that are only 1--10 pixels across.
In natural images, features at both spatial scales are important: you cannot remove features at either scale without radically altering the image \citep{van1996modelling}.

For attention over text, instead of pixels, we considered {\it token ranges\/} at different scales:
\begin{itemize}[topsep=0pt, parsep=0pt, itemsep=1.5pt]
  \item 1--10 tokens in the past (e.g.\ those in the same sentence),
  \item 10--100 tokens in the past (e.g.\ those in the same paragraph),
  \item 100--1,000 tokens in the past (e.g.\ those in the same section),
  \item 1,000--10,000 tokens in the past (e.g.\ those in the same document), and so on.
\end{itemize}
With this in mind, we define scale-invariant attention as any attention scheme that satisfies two desiderata: scale-invariant total attention and scale-invariant attention sparsity.

\begin{figure}[t]
    \centering
    \includegraphics[width=\linewidth]{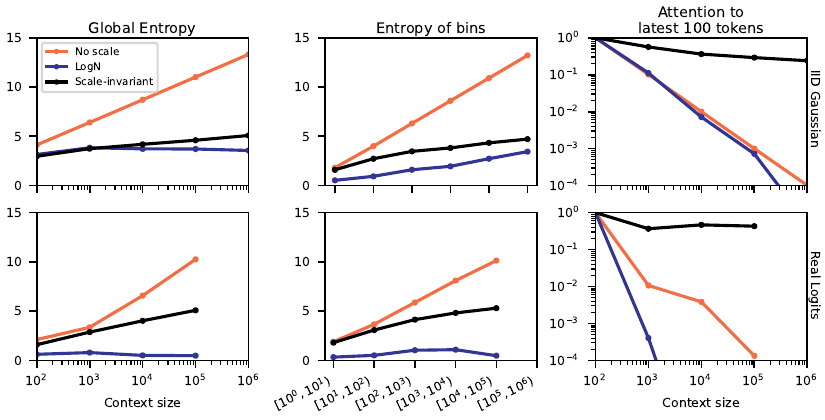}\\[-1.5ex]
    \caption{
    \textbf{Scale-invariant attention controls the entropy without sacrificing attention over the local context.} We consider three metrics for attention schemes: (\emph{left}) the global attention entropy, (\emph{middle}) entropy within particular ranges of tokens (e.g.\ 10--100), and (\emph{right}) total attention to the previous 100 tokens.
    The top row uses IID Gaussian logits, following our theoretical approach in Sec.~\ref{sec:logit_characteristics}. For LogN, the IID logits are multiplied by $s\log N$, where $N$ is the sequence length and $s=0.4$.
    The bottom row uses attention logits sampled from models trained with $p$-RoPE and  `No scale', LogN, and our scale-invariant transform.
    With no logit scaling, the attention becomes increasingly diffuse as the context grows (i.e. the distribution over logits has high entropy).
    LogN scaling reduces the entropy and thus ensures that attention remains sparse even at longer contexts.  However, LogN still forfeits the ability to attend to the local context (e.g.\ 100 most recent) tokens. Scale-invariant attention strikes a balance between low entropy and attending over the local context.}
    \label{fig:entropy}
\end{figure}

\textbf{Scale-invariant total attention} is the property that the sum of attention weights in each of the above ranges is roughly similar.
Intuitively, that means that the model attends to both the local context (e.g.\ 10--100 tokens ago) while at the same time taking account of information from the global context (e.g.\ 1,000--10,000 tokens ago).
Scale-invariant total attention addresses a key issue with attention mechanisms: as the context gets longer, models tend to pay more attention to distant tokens at the expense of the local context (e.g. Fig.~\ref{fig:entropy}, right column; the blue and orange lines decay quickly to zero).
While scale-invariant total attention doesn't entirely eliminate that issue, it does ensure that the attention paid to the local context shrinks only very slowly as the context length grows (e.g. Fig.~\ref{fig:entropy}, right column; the black lines decay to zero much more slowly).

Total attention tells us the amount of attention a certain region receives.
However, scale invariance of the total attention does not tell us about how the attention will be distributed among tokens in the region, i.e.\ whether the attention is spread out among many tokens or concentrated onto only a few tokens.
As the region gets wider (e.g.\ from 10--100 tokens ago to 10,000--100,000 tokens ago),
we might expect attention to spread out over more tokens simply due to the increased number of tokens.
This ``spreading out'' of attention is possibly suboptimal for large contexts~\citep{nakanishi2025scalable}, and that instead, we should focus attention on only a few of the most relevant tokens.
To measure these effects, we use the entropy, which roughly captures the logarithm of the number of tokens attended to.

\textbf{Scale-invariant attention sparsity} captures this notion.
In particular, we define two kinds of scale-invariant attention sparsity.
\textbf{Strong} scale-invariant attention sparsity implies that the number of tokens attended to in each region is constant.
For example, if the model attends to 8 tokens in the 10--100 token range, strong scale-invariant attention sparsity says it will also attend to approximately 8 tokens in the 1,000--10,000 token range.
Strong scale-invariant attention sparsity implies an extreme increase in sparsity as the context gets longer that may be difficult to achieve with practical attention mechanisms.
We therefore also introduce \textbf{weak} scale-invariant attention sparsity, which simply states that the sparsity increases as the context gets longer (i.e.\ attention is relatively dense in the region from 10--100 tokens, and much sparser in the region from 1,000--10,000, but you still attend to more tokens in the 1,000--10,000 region than the 10--100 region).

Our main contributions are:
\begin{itemize}[leftmargin=*, topsep=0pt, parsep=0pt, itemsep=1.5pt]
\item We introduce the concepts of scale-invariant total attention, and weak and strong scale-invariant attention sparsity as desirable properties when attending over long contexts.
\item We derive a simple, position-dependent transformation of attention logits that provably satisfies scale-invariant total attention, and empirically satisfies weak scale-invariant attention sparsity for Gaussian logits.
\item We implement scale-invariant attention in conjunction with $p$-RoPE~\citep{barbero2024round}. We show that our method, `scale-invariant $p$-RoPE', exhibits improvements in validation loss both when doing long-context training, and when zero-shot generalising to longer contexts. Our method also matches the performance of the best alternatives in an out-of-distribution long context `needle-in-a-haystack' task.
\end{itemize}
\section{Related work}
Handling long sequences effectively in Transformer-based models remains a significant challenge and is of high interest to the deep learning community~\citep{bai2024longalign,ye2024mplug,jin2024llm,beltagy2020longformer,ding2023longnet,munkhdalai2024leave,bulatov2024beyond,liu2023ring,barbero2024transformers,hu2024longrecipe}. Below, we discuss several strategies in this literature relating to our work.

{\bf Long contexts via the positional encoding}: Methods like ALiBi~\citep{press2021train} introduce a static, causal bias directly into the attention logits. ALiBi endows the model with an inductive bias towards recent tokens~\citep{kazemnejad2023impact}, thus helping with longer contexts. Rotary Position Embeddings (RoPE)~\citep{su2024roformer} have emerged as the most popular position encoding scheme, encoding relative positions. RoPE does not generalise to longer sequences out of the box~\citep{peng2023yarn}, leading to subsequent research improving RoPE's length extrapolation capabilities~\citep{zhu2023pose,wang2024resonance}. Many achieve this by modifying RoPE's frequency spectrum; for example, Positional Interpolation (PI)~\citep{chen2023extending}, NTK-aware scaling~\citep{bloc97_2023}, YaRN~\citep{peng2023yarn}, and by simply increasing the RoPE base $\theta$ parameter~\citep{grattafiori2024llama,team2025gemma,barbero2024round,liu2023scaling}.

{\bf Entropy Control}: Beyond positional information, the properties of the attention distribution itself are crucial, especially in long contexts where attention scores can ``smear''/``spread out'' across many tokens.
Scalable-Softmax (SSMax), also known as the `LogN trick'~\citep{nakanishi2025scalable,chiang2022overcoming,kexuefm-8823,bai2023qwen,llama4team2025}, addresses this by multiplying the attention logits by $s\log N$, where $N$ is the context length and $s$ is a learned scale parameter. This multiplier has the effect of sharpening/focusing the attention distribution.
\citet{li2025information} adopt a similar approach: controlling the entropy of the attention distribution for better length generalization.

These prior works are perhaps the most similar to our approach.
However, the key issue with such approaches is that they treat the local context (e.g.\ previous 10--100 tokens) in the same way as the global context (e.g.\ 10,000 tokens ago).
As the context gets larger, the attention paid to the 100 most recent tokens drops rapidly for the prior approaches, but stays markedly more consistent with scale-invariant attention (e.g.\ for LogN see Fig.~\ref{fig:entropy}).
In contrast, we started off by carefully specifying how we wanted attention to behave in the local and global contexts (Sec.~\ref{sec:defs}) by giving the scale-invariant total attention and scale-invariant attention sparsity desiderata. This means that e.g. LogN has a position-independent multiplicative bias, whereas our approach has position-dependent multiplicative and additive biases for the logits.

{\bf Efficient Long-Context Training and Inference}:
A key observation of~\cite{xiong2023effective} is that continual pretraining on long contexts, after initial pretraining on shorter sequences, is often sufficient and much more computationally efficient.
Thus, a simple strategy for improving long-context performance, given sufficient computational budget, is continual pretraining on longer contexts.
This has been adopted widely~\citep{grattafiori2024llama,gao2024train,lieber2024jamba,yang2024qwen2,cohere2025command,llama4team2025,liu2024deepseek}.
Of course, this strategy considerably increases the complexity and memory cost of the pretraining pipeline, making approaches like ours that can zero-shot generalise very valuable.
Furthermore, one might expect long-context training to be far easier (e.g.\ requiring fewer long-context training steps for optimal performance) for approaches that already have good long-context performance due to zero-shot generalisation.

For inference on sequences exceeding the trained context length, alternative strategies bypass applying dense attention over the entire context, allowing for `infinite attention'~\citep{munkhdalai2024leave,liu2024reattention,martins2021infty,chen2025edgeinfinite,ding2023longnet}. These include maintaining a fixed-size attention window~\citep{beltagy2020longformer}, retrieving relevant context tokens using heuristics like top-K proximity~\citep{han2023lm}, or vector similarity search akin to episodic memory systems~\citep{fountas2024human,xiao2024infllm}. While effective, these approaches operate at a higher level, managing context rather than modifying the core attention mechanism's ability to process it directly, which is the focus of our work.

\section{Methods}
\label{sec:methods}
Following FlexAttention \citep{dong2024flex}, we define the attention ``score'' as the dot product of a query $q$ and keys $K$,
\begin{align}
  \label{eq:def:S}
  S_t &= \frac{1}{\sqrt{d}} \sum_{\lambda=1}^d q_{\lambda} K_{t\lambda}.
\end{align}
Here, $d$ is the head dimension and $\lambda\in\{1, \ldots, d\}$ indexes the feature, and $t$ indexes the position.
We are using an unusual form for the sequence indices, but this simplifies notation later.
Specifically, we consider a single, fixed query token $q$.
Then, $t\geq 0$ into the keys/values from previous tokens.
Critically, $t$ counts backwards from the query token, e.g. $t=1$ indicates the previous token.
Logits are computed by applying an attention modifier function to the score,
\begin{align}
  \label{eq:def:L}
  L_t &= L(S_t; t).
\end{align}
Here, $L_t$ is the actual value of the attention logits for the token $t$ steps back from the query, while $L(S_t; t)$ is the function used to compute those logits.
Thus, $L(S_t; t) = S_t$ recovers standard attention.

For use later, we define the unnormalised attention weights $\tilde{A}_t$, (normalised) attention weights $A_t$, and normalisers $Z$, for a sequence of length $T$, as,
\begin{align}
  \label{eq:def:AAZ}
  \tilde{A}_t &= \exp\b{L_t} &
  A_t &= \frac{\tilde{A}_t}{Z} &
  Z &= \sum_{t=1}^T \tilde{A}_t.
\end{align}
\subsection{Formal definitions of scale-invariant total attention and attention sparsity}
\label{sec:defs}
\textbf{Scale-invariant total attention.}
In our examples we have considered ranges of tokens, 10--100, 100--1,000, 1,000--10,000, 10,000--100,000, etc..
Scale-invariant total attention is the property that the total attention in each of these ranges is somewhat similar, such that the attention allotted to any one range does not dominate the others. We give a formal definition in Def.~\ref{def:scale_inv_total_attn}, where using $\Delta=10$ gives the ranges from the examples.
\begin{definition}[Scale-invariant total attention]\label{def:scale_inv_total_attn}
Consider a set of random variables, $\{L_t\}$, representing attention logits.
The total unnormalised attention in range $t_1$ to $t_2$ is,
\begin{align}
  \label{eq:def:Z}
  Z_{t_1}^{t_2} &= \sum_{t=t_1}^{t_2-1} \tilde{A}_t = \sum_{t=t_1}^{t_2-1} \exp(L_t).
\end{align}
We say that the total attention is scale-invariant if, for any integer $\Delta > 1$, the expected total unnormalised attention in the range $\{t, t+1,\ldots, t\Delta - 1\}$ is asymptotically constant. That is,
\begin{align}
  \label{eq:scale_inv_Z}
  \E{Z_{t}^{t \Delta}} = \Theta(1) \quad \text{ as } \quad t \rightarrow \infty.
\end{align}
\end{definition}
The ``$\Theta(1)$'' is big-$\Theta$ notation, and means there exist constants $c_1, c_2 > 0$ and $t_0$ such that for all $t > t_0$, $c_1 \leq \E{Z_t^{t\Delta}} \leq c_2$.

\textbf{Scale-invariant unnormalised attention sparsity.}
Remember that scale-invariant attention sparsity means that attention should be denser for the local context (e.g.\ 10--100 tokens ago) and sparser for the global context (e.g.\ 1,000--10,000) tokens ago.
To measure attention sparsity, we consider the number of tokens attended to in a region.
To evaluate the number of tokens attended to, one approach is to use the entropy over tokens.
We measure the sparsity in the region from $t_1$ to $t_2$, using the entropy for the distribution over tokens in this region,
\begin{align}
  \label{eq:def:Htt}
  H_{t_1}^{t_2} &= - \sum_{t=t_1}^{t_2-1} \frac{\tilde{A}_t}{Z_{t_1}^{t_2}} \log \left(\frac{\tilde{A}_t}{Z_{t_1}^{t_2}}\right) = - \frac{\sum_{t=t_1}^{t_2-1} \tilde{A}_t \log \tilde{A}_t}{Z_{t_1}^{t_2}} + \frac{\log Z_{t_1}^{t_2} \sum_{t=t_1}^{t_2-1} \tilde{A}_t}{Z_{t_1}^{t_2}},
\end{align}
where $A_t$ is defined in Eq.~\eqref{eq:def:AAZ}.
Defining the unnormalised negentropy as,
\begin{align}
  \label{eq:def:N}
  \tilde{N}_{t_1}^{t_2} = \sum_{t=t_1}^{t_2-1} \tilde{A}_t \log \tilde{A}_t,
\end{align}
and remembering the definition $Z_{t_1}^{t_2}$ (Eq.~\ref{eq:def:Z}), we can then write the entropy in range $t_1$ to $t_2$ as,
\begin{align}
  \label{eq:def:Htt_negent}
  H_{t_1}^{t_2} &= - \frac{\tilde{N}_{t_1}^{t_2}}{Z_{t_1}^{t_2}} - \log Z_{t_1}^{t_2}.
\end{align}
Thus, it seems reasonable to consider the behavior of the unnormalised negentropy (Eq.~\ref{eq:def:N}), because if $\tilde{N}_{t}^{t \Delta}$ and $Z_{t}^{t \Delta}$ are asymptotically constant as $t \rightarrow \infty$, then by Eq.~\eqref{eq:def:Htt_negent} we expect $H_{t_1}^{t_2}$ to also be asymptotically constant.
Thus we define:

\begin{definition}[Scale-invariant unnormalised attention sparsity]\label{def:scale_inv_unnorm_attn_sparsity}
Consider a set of random variables, $\{L_t\}$, representing attention logits. We say that we have scale-invariant unnormalised attention sparsity if, for any integer $\Delta>0$,
\begin{align}
  \E{\tilde{N}_{t}^{t \Delta}} = \Theta(1) \quad \text{ as } \quad t \rightarrow \infty.
\end{align}
where $\tilde{N}_{t}^{t \Delta}$ is the unnormalised negentropy (Eq.~\ref{eq:def:N}).
\end{definition}
\begin{figure}[t]
    \centering
    \includegraphics[width=\linewidth]{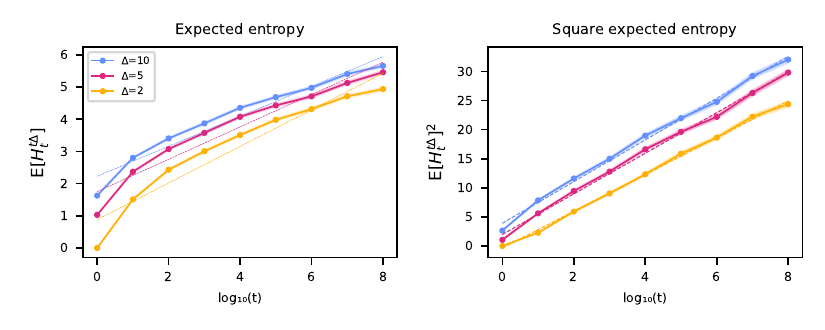}\\[-1.8ex]
    \caption{\textbf{Expected entropy of scale-invariant attention at different scales is sub-logarithmic}. Here, we sample sequences of independent standard Gaussian logits, and apply the scale-invariant attention transformation. We estimate the expected entropy in ranges $[t, t\Delta)$, where the size of the range is controlled by $t$ (x-axis) and $\Delta$ (line color). We see that this expected entropy measure scales sub-logarithmically (left), and with the right plot suggesting a $\sim \sqrt{\log(t)}$ scaling. The dashed lines show a best linear fit.}
    \label{fig:sublog}
    \vspace{-0.35cm}
\end{figure}
\textbf{Weak and strong scale-invariant attention sparsity.}
While the argument above suggests that scale-invariant unnormalised attention sparsity is an important property, ultimately we are interested in giving formal definitions of weak and strong attention sparsity.

We define weak scale-invariant attention sparsity such that as input lengths increase --- for example, from 10--100 to 100--1,000 tokens --- the number of attended tokens grows sublinearly. In contrast, standard attention with unscaled logits yields linear growth. Since entropy is roughly the log of the number of attended tokens (a uniform distribution over $k$ tokens has entropy equal to $\log k$), weak sparsity requires sublinear growth in entropy with respect to $\log(t)$ (see Definition~\ref{def:weak_scale_inv_attn_sparsity}).

\begin{definition}[Weak scale-invariant attention sparsity]\label{def:weak_scale_inv_attn_sparsity}
Consider a set of random variables, $\{L_t\}$, representing attention logits.
We say that the attention sparsity is weakly scale-invariant if, for any integer $\Delta > 1$,
\begin{align}
  \label{eq:weak_scale_inv_H}
  \E{H_{t}^{t \Delta}} &= o(\log t) \quad \text{ as } \quad t \rightarrow \infty.
\end{align}
\end{definition}
Remember $ o(\log(t))$ is `little-$o$' notation which means that $\E{H_t^{t \Delta}}$ scales strictly slower than $\log(t)$.
Strong scale-invariant attention sparsity implies that the number of tokens attended to is asymptotically constant as we go from e.g.\ the past 10--100 tokens to the past 1,000--10,000.

\begin{definition}[Strong scale-invariant attention sparsity]\label{def:strong_strong_scale_inv_attn_sparsity}
Consider a set of random variables, $\{L_t\}$, representing attention logits.
We say that the attention sparsity is strongly scale-invariant if, for any integer $\Delta > 1$,
\begin{align}
  \label{eq:strong_scale_inv_H}
  \E{H_{t}^{t \Delta}} &= \Theta(1) \quad \text{ as } \quad t \rightarrow \infty,
\end{align}
i.e.\ we expect $H_{t}^{t \Delta}$ to be asymptotically constant as $t \rightarrow \infty$.
\end{definition}
\subsection{What characteristics of the logits are required for scale-invariant attention?}\label{sec:logit_characteristics}
Next, we ask what properties would be sufficient for scale-invariant total attention and for some form of scale-invariant attention sparsity.
Lemma~\ref{lemma:total_attn} gives scale-invariant total attention, and Lemma~\ref{lemma:unnormalized_sparsity} gives scale-invariant unnormalised attention (see Appendix~\ref{app:proofs_lemma1_2} for the proofs).
\begin{restatable}{lemma}{totalattnlemma}\label{lemma:total_attn}
Consider a set of random variables, $\{L_t\}$, representing attention logits. Let $\tau>0$ be a lengthscale parameter, and $\alpha>0$ a multiplicative constant.
If the attention logits satisfy,
\begin{align}
  \label{eq:EA}
  \E{\tilde{A}_t} = \frac{\alpha}{t/\tau + 1},
\end{align}
then we have scale-invariant total attention (Def.~\ref{def:scale_inv_total_attn}).
\end{restatable}

\begin{restatable}{lemma}{unnormalizedsparsitylemma}\label{lemma:unnormalized_sparsity}
Consider a set of random variables, $\{L_t\}$, representing attention logits. Let $\tau>0$ be a lengthscale parameter, and $\beta>0$ a multiplicative constant.
If the attention logits satisfy,
\begin{align}
  \label{eq:EAlogA}
  \E{\tilde{A}_t \log \tilde{A}_t} = \frac{\beta}{t/\tau + 1},
\end{align}
then we have scale-invariant unnormalised attention sparsity (Def.~\ref{def:scale_inv_unnorm_attn_sparsity}).
\end{restatable}
To construct an attention mechanism that satisfies Eq.~\eqref{eq:EA} and Eq.~\eqref{eq:EAlogA}, we consider a simplified setting in which the logits are marginally Gaussian and arise from taking Gaussian ``base logits'', $\bar{L}_t \sim \N{0, 1}$, and transforming them by multiplying by $a_t$ and adding a bias, $m_t$,
\begin{align}\label{eq:transformed_gaussian_logits}
  L_t &= a_t \bar{L}_t + m_t\sim \mathcal{N}(m_t, a_t^2).
\end{align}
Our goal is to find $m_t$ and $a^2_t$ such that scale-invariant total attention and scale-invariant unnormalised attention sparsity hold.
In particular that requires,
\begin{subequations}\label{eq:scale_inv_attn_eqns}
\begin{align}
  \label{eq:gaussian_total}
  \frac{\alpha}{\tfrac{t}{\tau}+1} &= \E{\tilde{A}_t} = e^{m_t + a_t^2/2},\\
  \label{eq:gaussian_sparse}
  \frac{\beta}{\tfrac{t}{\tau}+1} &= \E{\tilde{A}_t \log \tilde{A}_t} = (m_t + a_t^2) e^{m_t + a_t^2/2},
\end{align}
\end{subequations}
where $\tilde{A}_t = \exp\b{L_t}$.
Solving for $m_t$ and $a_t$, we have (see Appendix~\ref{app:solve_mean_var} for details),
\begin{subequations}
\label{eq:solution}
\begin{align}
\label{eq:gaussian_at}
a_t &= \sqrt{2\left[\log(t/\tau + 1) - \log \alpha + \beta / \alpha\right]},\\
\label{eq:gaussian_mt}
m_t &= -a^2_t + \beta/\alpha.
\end{align}
\end{subequations}
For this solution to be valid, we only require $\beta \geq \alpha \log \alpha$ since $\log (t/\tau + 1)\geq 0$ when $t \geq 0$.
We formally summarise the above results in Theorem~\ref{theorem:scale_inv} (proof in Appendix~\ref{app:proof_theorem1}), which tells us that this approach does indeed give scale-invariant total attention and scale-invariant unnormalised attention sparsity.
\begin{restatable}{theorem}{thmscaleinv}\label{theorem:scale_inv}
Suppose attention logits $\{L_t\}$ are marginally Gaussian with mean $m_t$ and standard deviation $a_t$ defined by Eq.~\eqref{eq:solution}.
Assuming $\alpha,\beta,\tau > 0$, $\beta\geq \alpha\log\alpha$, then we have scale-invariant total attention and scale-invariant unnormalised  attention sparsity.
\end{restatable}
We therefore propose to scale the logits in real attention using $a_t$ and $m_t$ defined in Eq.~\eqref{eq:solution}.
As $t$ increases, the variance, $a_t^2$, \textit{increases} as the logarithm of $t$, while the mean \textit{decreases} as the logarithm of $t$.

Finally, note that while we have proven that we have scale-invariant unnormalised attention with this choice of $a_t$ and $m_t$, we have not proven that we have strong or weak scale-invariant attention.
We therefore checked empirically whether using IID Gaussian logits, scaled by $a_t$ and $m_t$, gave weak or strong scale-invariant attention entropy (Fig.~\ref{fig:sublog}).
We find that $H_{t}^{t \Delta}$ appears to scale with $\sqrt{\log(t)}$, and hence seems to satisfy weak, but not strong scale-invariant attention sparsity.

\textbf{Hyperparameters. } Introducing $a_t$ and $m_t$ (Eq.~\ref{eq:solution}), appears to have introduced three additional hyperparameters to tune. We reduce this to one additional hyperparameter, the lengthscale $\tau$, by specifying a boundary condition --- $a_{0}^2 = 1$ and $m_{0} = 0$. This boundary condition corresponds to not changing the scale of the local tokens.
Substituting into Eq.~\eqref{eq:gaussian_mt} gives $0 = -1 +\beta/\alpha$. Similarly Eq.~\eqref{eq:gaussian_at} requires that $1 = -2 \log \alpha + 2\beta/\alpha$. Putting these together, we obtain $\alpha = \beta = \mathrm{e}^{0.5}$.

That leaves us with the lengthscale as the only hyperparameter.
Note that when $t$ is small relative to $\tau$, neither $a_t^2$ nor $m_t$ change much. Therefore, the timescale sets the size of a local region in which attention is approximately unscaled.
Intuitively, it therefore makes sense to choose $\tau$ somewhere in the region of 1--100 tokens. In Appendix~\ref{app:setting_lengthscale} we tried $\tau\in\{10^{-2}, 10^{-1}, 10^0, 10^1, 10^2\}$ and find that  $\tau=10$ performs best in practice.
\begin{figure}
    \centering
    \begin{subfigure}[h]{\linewidth}
        \includegraphics[width=\linewidth]{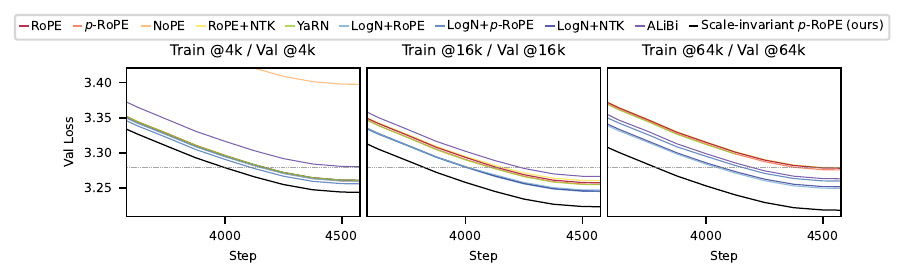}
        \vspace{-0.5cm}
        \caption{Scale-invariant attention improves language modelling over a range of training context lengths.}
        \label{fig:pretrain_various_lengths}
    \end{subfigure}
    \vspace{1ex}
    \begin{subfigure}[h]{\linewidth}
        \includegraphics[width=\linewidth]{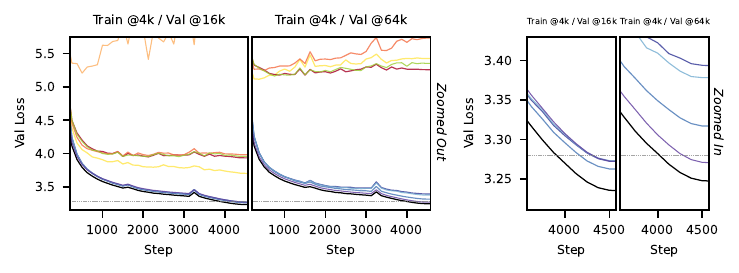}
        \vspace{-0.5cm}
        \caption{Scale-invariant attention enables zero-shot long-context generalization of at least 16x.}
        \label{fig:pretrain_4k}
    \end{subfigure}
    \caption{Validation losses throughout training of a 162M parameter GPT-2-like model with different attention mechanisms (our scale-invariant scheme shown in black). NoPE is omitted in all but top-left and bottom-left panels to avoid excessive zooming out (due to high loss). The gray dashed line shows the baseline of 3.28. The validation loss for many methods increases in unison at steps $\sim$1500 and $\sim$3250 (see (b), left) despite aggregating over seeds; this is due to a fixed training and validation data ordering.}
    \label{fig:pretrain_combined}
    \vspace{-0.35cm}
\end{figure}
\section{Experiments}\label{sec:experiments}
In this section, we compare scale-invariant attention with other dense attention methods including Dynamic NTK interpolation (RoPE+NTK)~\citep{bloc97_2023}, LogN scaling/SSMax~\citep{nakanishi2025scalable}, $p$-RoPE~\citep{barbero2024round}, and ALiBi~\citep{press2021train}. Our results show that our method, scale-invariant $p$-RoPE, has uniformly lower validation loss at a variety of training lengths (4k, 16k, 64k). Additionally, scale-invariant $p$-RoPE demonstrates stronger zero-shot long-context generalization (e.g. in `Train @4k/Val @16k' and `Train @4k/Val @64k' settings) versus all other methods. Finally, scale-invariant~\prope, along with LogN, saturated``needle-in-a-haystack'' task.

 We pretrained GPT-2-style models~\citep{radford2019language} (with QK-norm, ReLU$^2$ activations, etc.~\citep{modded_nanogpt_2024}) from scratch on the FineWeb dataset~\citep{penedo2024fineweb}, using a fixed training data ordering and a 10M token validation set. We trained linear layers with Muon~\citep{jordan2024muon}, and remaining parameters with Adam. We implemented scale-invariant attention using FlexAttention~\citep{dong2024flex}.

We trained two models: one with 162M parameters and another with 304M. The 162M one were trained for 4578 steps on 2.4B tokens over a range of context lengths (4k, 16k, 64k). The 304M parameter models were trained only on the best length-generalising methods, for 10.9k steps on 10B tokens, at 4k context length. We give further experimental details in Appendix~\ref{app:further_exp}. For the 162M parameter model, we targeted a validation loss of 3.28, following Karpathy’s GPT-2 reproduction~\citep{karpathy_llmc_481,modded_nanogpt_2024}, shown in the figures as a horizontal grey dashed line.

\textbf{Long-context performance. } We examined in-distribution, long-context performance of the different attention methods by looking at the validation loss for the same context length used for training. Even in this in-distribution setting, scale-invariant~\prope~shows strong improvements in validation loss at all training lengths, 4k, 16k, and 64k (Fig.~\ref{fig:pretrain_various_lengths}).

\textbf{Length Generalization. } We evaluate length generalization by measuring validation loss on long sequences (16k and 64k) when training on 4k tokens. The `Train @4k/Val @64k' setting in particular represents a considerable jump of 16× between train and validation. Table~\ref{tab:final_losses} and Fig.~\ref{fig:pretrain_4k} report results for the 162M model. In the left panel of the Figure, ALiBi, LogN, and our method substantially outperform other approaches in generalizing to longer sequences. The right panel zooms in and shows that scale-invariant~\prope~achieves the strongest generalization overall.

In preliminary experiments we tried other scale-invariant methods. We found that the most obvious method, scale-invariant RoPE, did not generalise well to long contexts (see Appendix~\ref{app:scale_invariant_comparisons}). The $p$-RoPE method is similar to RoPE but excludes low-frequency/high-wavelength components in the position embedding, possibly suggesting that low-frequency components in RoPE interfere with the scale-invariant transformation $L_t \mapsto a_t L_t + m_t$.
LogN also demonstrates stronger performance when paired with $p$-RoPE  rather than RoPE.
We were surprised that RoPE+NTK struggled to generalise in the `Train @4k / Val @64k' setting, but we believe this can be explained by the training context size: `Train @16k / Val @64k' is much better for RoPE+NTK (see Fig.~\ref{fig:pretrain_16k} in the Appendix).

Fig.~\ref{fig:med_model_4k} presents pretraining losses for ALiBi, LogN+$p$-RoPE, and our method on a larger 304M model --- selected due to their performance on the 162M `Train @4k/Val @64k' task. Scale-invariant $p$-RoPE maintains its advantage at this larger scale.

\textbf{Needle in a Haystack. } The key benefit of scale-invariant attention is that it balances local and sparse global attention.
As such, we might worry that long-context retrieval performance might suffer versus other approaches that do not have specific mechanisms to ensure that attention to the local context does not vanish.

To assess whether long-context information retrieval capabilities suffered, we fine-tuned models on a needle-in-a-haystack task (note that fine-tuning on this task is unusual, but we found that prompting alone was not sufficient to perform this task, as bigger models/more pretraining would be required).
Needle-in-a-haystack~\citep{kamradt2023needle} measures a model's ability to precisely retrieve specific details (needles) from a large body of text (the haystack).
Our tasks constructs prompts by concatenating text samples from the C4 dataset~\citep{roberts2019exploring} and embedding ``needles'' of the form `\texttt{The special magic <city> number is <7\_digit\_number>}'. We insert three (rather than one) needles uniformly, at random, into each context for more signal per example, with each needle contributing separately to the overall accuracy.
A successful retrieval requires the model to output both the city and the associated number correctly. We trained models on sequences of length 4k, and tested at 4k, 16k, and 64k.

Table~\ref{tab:nih_val_acc} show that scale-invariant $p$-RoPE and LogN+\prope~perform well, while almost all other methods fail almost completely at 64k context length.
Thus, despite focusing more on local context, our method does not seem to have suffered in retrieval performance.
\begin{table}[t]
\small
\centering
\caption{Final mean validation losses ($\pm$1 standard error across 3 seeds) for different methods when training with 4k context length on a 162M parameter GPT-2-style model. The error bars are small due to consistent training data ordering, and a fixed validation set.}
\label{tab:final_losses}
\begin{tabular}{lccc}
\toprule
Method & Val @ 4k & Val @ 16k & Val @ 64k \\
\midrule
RoPE & 3.261 $\pm$ 0.001 & 3.936 $\pm$ 0.010 & 5.260 $\pm$ 0.014 \\
$p$-RoPE & 3.260 $\pm$ 0.001 & 3.984 $\pm$ 0.008 & 5.735 $\pm$ 0.085 \\
NoPE & 3.397 $\pm$ 0.000 & 6.430 $\pm$ 0.059 & 8.125 $\pm$ 0.062 \\
RoPE+NTK & 3.261 $\pm$ 0.001 & 3.703 $\pm$ 0.007 & 5.430 $\pm$ 0.026 \\
YaRN & 3.261 $\pm$ 0.000 & 3.958 $\pm$ 0.018 & 5.353 $\pm$ 0.065 \\
LogN+RoPE & 3.260 $\pm$ 0.001 & 3.273 $\pm$ 0.004 & 3.378 $\pm$ 0.011 \\
LogN+$p$-RoPE & 3.256 $\pm$ 0.001 & 3.262 $\pm$ 0.002 & 3.317 $\pm$ 0.005 \\
LogN+NTK & 3.261 $\pm$ 0.001 & 3.272 $\pm$ 0.002 & 3.394 $\pm$ 0.025 \\
ALiBi & 3.281 $\pm$ 0.001 & 3.272 $\pm$ 0.001 & 3.270 $\pm$ 0.000 \\
Scale-invariant $p$-RoPE (ours) & \textbf{3.244} $\pm$ 0.001 & \textbf{3.235} $\pm$ 0.001 & \textbf{3.247} $\pm$ 0.001 \\
\bottomrule
\end{tabular}
\end{table}

\begin{figure}[t]
    \centering
    \includegraphics[width=\linewidth]{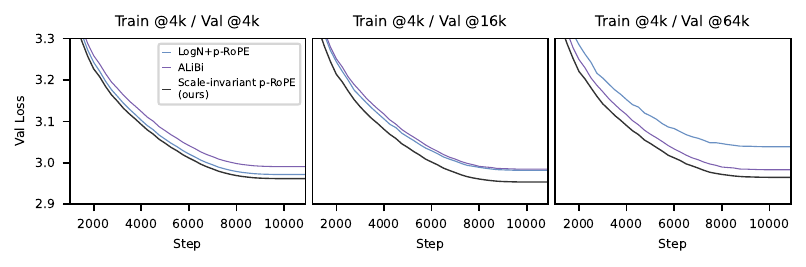}\\[-1.8ex]
    \caption{Validation losses throughout training for a 304M parameter model.}
    \label{fig:med_model_4k}
\end{figure}

\begin{table}[!t]
\centering
\caption{Mean validation accuracies on the needle-in-a-haystack task, $\pm$1 standard error. Metrics were calculated over 3 seeds, after 300 steps of fine-tuning.}
\label{tab:nih_val_acc}
\small
\begin{tabular}{lccc}
\toprule
Method & Val Acc @4k & Val Acc @16k & Val Acc @64k \\
\midrule
RoPE & 0.962 $\pm$ 0.003 & 0.000 $\pm$ 0.000 & 0.000 $\pm$ 0.000 \\
$p$-RoPE & 0.966 $\pm$ 0.001 & 0.250 $\pm$ 0.020 & 0.000 $\pm$ 0.000 \\
NoPE & 0.964 $\pm$ 0.000 & 0.303 $\pm$ 0.239 & 0.000 $\pm$ 0.000 \\
RoPE+NTK & 0.962 $\pm$ 0.001 & 0.217 $\pm$ 0.078 & 0.000 $\pm$ 0.000 \\
YaRN & \textbf{0.969} $\pm$ 0.001 & 0.000 $\pm$ 0.000 & 0.000 $\pm$ 0.000 \\
LogN+RoPE & 0.965 $\pm$ 0.002 & 0.276 $\pm$ 0.010 & 0.064 $\pm$ 0.006 \\
LogN+$p$-RoPE & \textbf{0.969} $\pm$ 0.002 & 0.962 $\pm$ 0.003 & 0.939 $\pm$ 0.009 \\
LogN+NTK & 0.962 $\pm$ 0.001 & 0.253 $\pm$ 0.015 & 0.056 $\pm$ 0.008 \\
ALiBi & 0.957 $\pm$ 0.002 & 0.020 $\pm$ 0.001 & 0.003 $\pm$ 0.000 \\
Scale-invariant $p$-RoPE (ours) & 0.965 $\pm$ 0.000 & \textbf{0.969} $\pm$ 0.004 & \textbf{0.969} $\pm$ 0.005 \\
\bottomrule
\end{tabular}
\end{table}
\section{Limitations}\label{sec:limitations}
In this work, we evaluated our methods by pretraining with 162M and 304M parameter models, and investigated 7B parameter models via continual pretraining in Appendix~\ref{app:larger_models}. Ideally, we would have pretrained from scratch at the multi-billion parameter scale, in line with contemporary state-of-the-art LLMs. While compute resource constraints prevented this, our results, together with the natural theoretical approach, offer no indication that our conclusions would fail to generalise to larger, commercially deployed models.

We focused on scale-invariant $p$-RoPE with dense attention, but the extension to other settings is a promising direction for future investigation.
\section{Conclusions}
We have proposed two desirable properties of attention mechanisms, ``scale-invariant total attention'' and ``scale-invariant attention sparsity'', and presented a straightforward modification to attention logits that enables these properties in practice. In our experiments, we found that our scale-invariant attention modification, especially when paired with $p$-RoPE, substantially improves long-context language modelling performance and gives zero-shot generalisation from training on short contexts to testing on long contexts.
\FloatBarrier
\section{Acknowledgements}
We sincerely thank the Engineering and Physical Sciences Research Council and the COMPASS CDT for funding Ben Anson at the University of Bristol. Many of our experiments used computational resources at the Advanced Computing Research Centre at the University of Bristol, with GPUs generously funded by Dr. Stewart.
\bibliography{refs}
\bibliographystyle{icml2024}

\newpage
\appendix
\section{Relating the Indicies in the Paper to Standard Indices}
The standard form for attention is,
\begin{align}
  S^{ij} &= \frac{1}{\sqrt{d}} \sum_{\lambda=1}^d Q_{i\lambda} K_{j\lambda},
\end{align}
where $\lambda$ indexes the feature, $i$ indexes the query (i.e.\ the token we're generating now) and $j$ indexes the key (i.e.\ the token we're attending to).
We fix $i$, and take $t = i-j$, so
\begin{align}
  S^{ij} &= S_{t=i-j}
\end{align}
where $S_{t=i-j}$ is given in Eq.~\ref{eq:def:S} in the main text.
Then,
\begin{align}
  L^{ij} &= L(S^{ij}; i-j) = L(S_t; t) = L_{t=i-j},
\end{align}
where $L_{t=i-j}$ is given in Eq.~\ref{eq:def:S} in the main text.
Then the unnormalised attention weights, $\tilde{A}^{ij}$, normalised attention weights, $A^{ij}$ and normalisers, $Z^i$ are,
\begin{align}
  \tilde{A}^{ij} &= \exp\b{L^{ij}},\\
  \intertext{where, $\tilde{A}_{t=i-j}$ is given in Eq.~\ref{eq:def:AAZ} in the main text.}
  A^{ij} &= \frac{\tilde{A}^{ij}}{Z^i}, \\
  \intertext{where, $A_{t=i-j}$ is given in Eq.~\ref{eq:def:AAZ} in the main text.}
  Z^i &= \sum_{j=1}^i\tilde{A}^{ij} = Z,
\end{align}
where, $Z$ is given in Eq.~\ref{eq:def:AAZ} in the main text.
\section[Expected value of powers of X multiplied by exponential of X, for Gaussian X]{$\E{X^k \exp(\alpha X)}$ where $X$ is Gaussian}\label{app:integral}
Assume $X \sim \mathcal{N}(\mu, \sigma^2)$, then we can compute the expectation in terms of a moment of a Gaussian,
\begin{align}
\mathbb{E}[X^k e^{\alpha X}] &= \int_{-\infty}^{\infty} x^k e^{\alpha x} \frac{1}{\sqrt{2\pi\sigma^2}} e^{-\frac{(x-\mu)^2}{2\sigma^2}} dx \\
&= \int_{-\infty}^{\infty} x^k \frac{1}{\sqrt{2\pi\sigma^2}} e^{\alpha x - \frac{(x-\mu)^2}{2\sigma^2}} dx \\
&= \int_{-\infty}^{\infty} x^k \frac{1}{\sqrt{2\pi\sigma^2}} e^{-\frac{1}{2\sigma^2}(x^2 - 2\mu x + \mu^2 - 2\alpha\sigma^2 x)} dx \\
&= \int_{-\infty}^{\infty} x^k \frac{1}{\sqrt{2\pi\sigma^2}} e^{-\frac{1}{2\sigma^2}(x^2 - 2(\mu + \alpha\sigma^2)x + \mu^2)} dx \\
&= e^{\frac{\alpha^2\sigma^2}{2} + \alpha\mu} \int_{-\infty}^{\infty} x^k \frac{1}{\sqrt{2\pi\sigma^2}} e^{-\frac{(x-(\mu+\alpha\sigma^2))^2}{2\sigma^2}} dx\\
&= e^{\frac{\alpha^2\sigma^2}{2} + \alpha\mu} \E{\tilde{X}^k},
\intertext{where $\tilde{X}\sim \mathcal{N}(\mu', \sigma'^2)= \mathcal{N}(\mu + \alpha\sigma^2, \sigma^2)$. In the case $\alpha = k = 1$, we have,}
\E{X\exp(X)}&= (\mu + \sigma^2)\exp(\mu + \sigma^2/2).
\intertext{In the case $\alpha = k = 2$, we have,}
\E{X^2\exp(2X)} &= ((\mu + 2\sigma^2)^2 + \sigma^2)\exp(2\mu + 2\sigma^2).
\end{align}

\section[Deriving logit transformation]{Deriving $a_t$ and $m_t$ in the Logit Transformation}
\label{app:solve_mean_var}
From Eq.~\eqref{eq:scale_inv_attn_eqns}, we wish to find $a_t$ and $m_t$ such that,
\begin{align}
  \frac{\alpha}{\tfrac{t}{\tau}+1} &= e^{m_t + a_t^2/2} \label{eq:app:gaussian_total}\\
  \frac{\beta}{\tfrac{t}{\tau}+1} &= (m_t + a_t^2) e^{m_t + a_t^2/2}. \label{eq:app:gaussian_sparse}
\end{align}
We begin by dividing Eq.~\eqref{eq:app:gaussian_sparse} by Eq.~\eqref{eq:app:gaussian_total},
\begin{align}
  \frac{\beta}{\alpha} &= m_t + a_t^2, \label{eq:app:gaussian_ratio}\\
  \intertext{which can be rearranged to,}
  m_t &= \frac{\beta}{\alpha} - a_t^2,\label{eq:app:gaussian_ratio:m}\\
\text{and }  a_t^2 &= \frac{\beta}{\alpha} - m_t. \label{eq:app:gaussian_ratio:a2}
\end{align}

Now, taking the log of Eq.~\eqref{eq:app:gaussian_total}, we have,
\begin{align}
  m_t + \tfrac{1}{2} a_t^2 &= - \log \b{\tfrac{t}{\tau}+1} + \log \alpha. \label{eq:app:gaussian_total_log}
  \end{align}
Define $f_t = \log \b{\tfrac{t}{\tau}+1} - \log \alpha$. Then
to solve for $a_t^2$, we substitute $m_t$ from Eq.~\eqref{eq:app:gaussian_ratio:m} into Eq.~\eqref{eq:app:gaussian_total_log},
\begin{align}
  \b{\tfrac{\beta}{\alpha} - a_t^2} + \tfrac{1}{2} a_t^2 &= - f_t\\
  \tfrac{\beta}{\alpha} - \tfrac{1}{2} a_t^2 &= - f_t\\
  a_t^2 &= 2 \b{f_t + \tfrac{\beta}{\alpha}}. \label{eq:app:gaussian_a2t_simple}
\end{align}

To solve for $m_t$, we substitute $a_t^2$ from Eq.~\eqref{eq:app:gaussian_ratio:a2} into Eq.~\eqref{eq:app:gaussian_total_log}
\begin{align}
  m_t + \tfrac{1}{2} \b{\tfrac{\beta}{\alpha} - m_t} &= - f_t\\
  \tfrac{1}{2} m_t + \tfrac{\beta}{2 \alpha} &= - f_t\\
  m_t &= - 2 \b{f_t + \tfrac{\beta}{2 \alpha}}, \label{eq:app:gaussian_mt_simple}
\end{align}
which can also be written as,
\begin{align}
  m_t &= - a_t^2 + \tfrac{\beta}{\alpha}.
\end{align}
Thus, we have,
\begin{align}
a_t^2 &=  2\left[\log(t/\tau + 1) + \beta/\alpha -\log \alpha\right]\\
m_t &=  -2\left[\log(t/\tau + 1) + \beta/\alpha -\log \alpha\right] + \beta/\alpha\\
    &= -2\log(t /\tau + 1) - \beta/\alpha + 2\log\alpha.
\end{align}
\section{Proofs of Lemmas \ref{lemma:total_attn} and \ref{lemma:unnormalized_sparsity}}\label{app:proofs_lemma1_2}
\totalattnlemma*
\begin{proof}
We want to show that $\E{Z_{t_1}^{t_1 \Delta}} = \Theta(1)$ as $t_1 \rightarrow \infty$.
By definition and linearity of expectation:
\begin{align}
  \E{Z_{t_1}^{t_1 \Delta}} &= \E{\sum_{t=t_1}^{t_1\Delta - 1} \tilde{A}_t} = \sum_{t=t_1}^{t_1\Delta - 1} \E{\tilde{A}_t}.
\end{align}
Using the given condition $\E{\tilde{A}_t} = \frac{\alpha}{t/\tau + 1}$:
\begin{align}
  \E{Z_{t_1}^{t_1 \Delta}} &= \sum_{t=t_1}^{t_1\Delta - 1} \frac{\alpha}{t/\tau + 1} = \alpha \tau \sum_{t=t_1}^{t_1\Delta - 1} \frac{1}{t + \tau} = \alpha \tau \sum_{k=t_1+\tau}^{t_1\Delta - 1 + \tau} \frac{1}{k}.
\end{align}
Using the standard integral bounds for the harmonic sum derived by comparison with the integral (see Appendix~\ref{app:harmonic_bounds}):
\begin{align}
  \ln\left(\frac{t_1\Delta + \tau}{t_1+\tau}\right) \le \sum_{k=t_1+\tau}^{t_1\Delta - 1 + \tau} \frac{1}{k} \le \ln\left(\frac{t_1\Delta - 1 + \tau}{t_1+\tau - 1}\right).
\end{align}
As $t_1 \rightarrow \infty$:
\begin{align}
  \frac{t_1\Delta + \tau}{t_1+\tau} &\rightarrow \frac{t_1\Delta}{t_1} = \Delta \\
  \frac{t_1\Delta - 1 + \tau}{t_1+\tau - 1} &\rightarrow \frac{t_1\Delta}{t_1} = \Delta
\end{align}
So, the logarithm terms in both the lower and upper bounds approach $\ln(\Delta)$. By the Squeeze Theorem:
\begin{align}
  \sum_{k=t_1+\tau}^{t_1\Delta - 1 + \tau} \frac{1}{k} \rightarrow \ln(\Delta)
\end{align}
Since $\alpha$, $\tau$, and $\Delta > 1$ are constants, $\ln(\Delta)$ is a positive constant. Thus:
\begin{align}
  \E{Z_{t_1}^{t_1 \Delta}} \rightarrow \alpha \tau \ln(\Delta) = \Theta(1)
\end{align}
This satisfies the condition for scale-invariant total attention in expectation.
\end{proof}

\unnormalizedsparsitylemma*
\begin{proof}
We want to show that $\E{\tilde{N}_{t_1}^{t_1 \Delta}} = \Theta(1)$ as $t_1 \rightarrow \infty$.
By definition and linearity of expectation:
\begin{align}
  \E{\tilde N_{t_1}^{t_1 \Delta}} &= \E{\sum_{t=t_1}^{t_1 \Delta-1} \tilde{A}_t \log \tilde{A}_t} = \sum_{t=t_1}^{t_1 \Delta-1} \E{\tilde{A}_t \log \tilde{A}_t}
\end{align}
Using the given condition $\E{\tilde{A}_t \log \tilde{A}_t} = \frac{\beta}{t/\tau + 1}$:
\begin{align}
  \E{\tilde{N}_{t_1}^{t_1 \Delta}} &= \sum_{t=t_1}^{t_1 \Delta-1} \frac{\beta}{t/\tau + 1} = \beta \tau \sum_{t=t_1}^{t_1 \Delta-1} \frac{1}{t + \tau}
\end{align}
This sum is exactly the same form as in the proof of Lemma~\ref{lemma:total_attn}, just with $\beta$ instead of $\alpha$. Following the same steps using the integral bounds for the harmonic sum derived in Appendix~\ref{app:harmonic_bounds}:
\begin{align}
  \sum_{k=t_1+\tau}^{t_1\Delta - 1 + \tau} \frac{1}{k} \rightarrow \ln(\Delta) \quad \text{ as } \quad t_1 \rightarrow \infty
\end{align}
Therefore, as $t_1 \rightarrow \infty$:
\begin{align}
  \E{\tilde{N}_{t_1}^{t_1 \Delta}} \rightarrow \beta \tau \ln(\Delta) = \Theta(1)
\end{align}
This satisfies the condition for scale-invariant unnormalised attention sparsity in expectation.
\end{proof}

\section{Bounds for Harmonic Sums}
\label{app:harmonic_bounds}

For the proofs of Lemma~\ref{lemma:total_attn} and Lemma~\ref{lemma:unnormalized_sparsity}, we need bounds on the partial harmonic sum $S = \sum_{k=a}^b \frac{1}{k}$, where $a = t_1+\tau$ and $b = t_1\Delta - 1 + \tau$. We can obtain these bounds by comparing the sum to the integral of $f(x) = 1/x$.

Since $f(x) = 1/x$ is a decreasing function for $x > 0$, we have,
\begin{align}
    \int_{a}^{b+1} \frac{1}{x} dx \le \sum_{k=a}^{b} \frac{1}{k} \le \int_{a-1}^{b} \frac{1}{x} dx
\end{align}
Evaluating the integrals gives,
\begin{align}
    \ln\left(\frac{b+1}{a}\right) &\le \sum_{k=a}^{b} \frac{1}{k} \le \ln\left(\frac{b}{a-1}\right). \label{eq:integral_bounds}
\end{align}
\section{Proof of Theorem~\ref{theorem:scale_inv}}\label{app:proof_theorem1}
\thmscaleinv*
\begin{proof}
We need to show that the given conditions are sufficient for scale-invariant total attention (Definition~\ref{def:scale_inv_total_attn}) and scale-invariant unnormalised attention sparsity (Definition~\ref{def:scale_inv_unnorm_attn_sparsity}).

\textbf{1. Scale-invariant total attention:}
We need to show that $\E{Z_{t_1}^{t_1 \Delta}} = \Theta(1)$ as $t_1 \rightarrow \infty$.
By Lemma~\ref{lemma:total_attn}, this holds if $\E{\tilde{A}_t} = \frac{\alpha}{t/\tau + 1}$.
Since $L_t \sim \mathcal{N}(m_t, a_t^2)$, the unnormalised attention weight $\tilde{A}_t = e^{L_t}$ follows a log-normal distribution. The expectation of a log-normal variable $e^X$ where $X \sim \mathcal{N}(\mu, \sigma^2)$ is $e^{\mu + \sigma^2/2}$.
Therefore,
\begin{align}
  \E{\tilde{A}_t} &= \E{e^{L_t}} = e^{m_t + a_t^2/2}
\end{align}
Substituting the given expressions for $m_t = -a_t^2 + \beta/\alpha$ and $a_t^2 = 2[\log(t/\tau + 1) - \log \alpha + \beta / \alpha]$:
\begin{align}
  \E{\tilde{A}_t} &= e^{(-a_t^2 + \beta/\alpha) + a_t^2/2} \\
  &= e^{-a_t^2/2 + \beta/\alpha} \\
  &= e^{-[\log(t/\tau + 1) - \log \alpha + \beta / \alpha] + \beta/\alpha} \\
  &= e^{-\log(t/\tau + 1) + \log \alpha - \beta / \alpha + \beta/\alpha} \\
  &= e^{\log \alpha - \log(t/\tau + 1)} \\
  &= e^{\log\left(\frac{\alpha}{t/\tau + 1}\right)} \\
  &= \frac{\alpha}{t/\tau + 1}
\end{align}
Since this condition matches the requirement of Lemma~\ref{lemma:total_attn}, scale-invariant total attention holds in expectation. The condition $\beta \ge \alpha \log \alpha$ ensures $a_t^2 \ge 0$ for all $t \ge 0$.

\textbf{2. Scale-invariant unnormalised attention sparsity:}
We need to show that $\E{\tilde{N}_{t_1}^{t_1 \Delta}} = \Theta(1)$ as $t_1 \rightarrow \infty$.
By Lemma~\ref{lemma:unnormalized_sparsity}, this holds if $\E{\tilde{A}_t \log \tilde{A}_t} = \frac{\beta}{t/\tau + 1}$.
We need to compute $\E{\tilde{A}_t \log \tilde{A}_t} = \E{L_t e^{L_t}}$.
Using the formula for $\E{X e^X}$ where $X \sim \mathcal{N}(\mu, \sigma^2)$ from Appendix~\ref{app:integral}, which is $(\mu + \sigma^2) e^{\mu + \sigma^2/2}$:
\begin{align}
  \E{L_t e^{L_t}} &= (m_t + a_t^2) e^{m_t + a_t^2/2}
\end{align}
Substitute $m_t = -a_t^2 + \beta/\alpha$:
\begin{align}
  \E{L_t e^{L_t}} &= (-a_t^2 + \beta/\alpha + a_t^2) e^{m_t + a_t^2/2} \\
  &= (\beta/\alpha) e^{m_t + a_t^2/2}
\end{align}
From the previous step, we know $e^{m_t + a_t^2/2} = \frac{\alpha}{t/\tau + 1}$. Substituting this:
\begin{align}
  \E{L_t e^{L_t}} &= (\beta/\alpha) \left( \frac{\alpha}{t/\tau + 1} \right) \\
  &= \frac{\beta}{t/\tau + 1}
\end{align}
Since this condition matches the requirement of Lemma~\ref{lemma:unnormalized_sparsity}, scale-invariant unnormalised attention sparsity holds.

Therefore, under the given conditions, we have both scale-invariant total attention and scale-invariant unnormalised attention sparsity.
\end{proof}
\section{Further experimental details}\label{app:further_exp}
We provide experimental details here, and we provide the code used in the supplementary materials.
\subsection{Pretraining from scratch}

Our base model is a \texttt{modded-nanogpt}~\citep{modded_nanogpt_2024} variant, which is similar to GPT-2~\citep{radford2019language}, but has the following main differences from GPT-2: RMSNorm for layer normalization (applied before the attention and MLP blocks, as well as to the token embeddings and the final output layer), squared ReLU activations, and QK Normalization (RMSNorm on query and key projections). All models used a vocabulary size of 50304, with text tokenized with the GPT-2 tokenizer~\citep{wolf2020transformers}.  We implemented scale-invariant attention and ALiBi using FlexAttention~\citep{dong2024flex}.

\textbf{162M parameter model. } The smaller model had 12 layers, 768 hidden dimension, and 6 heads. We optimized embedding parameters using Adam with learning rate $\gamma=0.3$, $\beta=(0.9, 0.95)$. We optimized linear layers with Muon, with no weight decay, $\gamma=0.02$ and momentum $0.95$. For remaining parameters (unembedding and LogN scalings, if LogN trick was active), we optimized with Adam, using $\gamma=0.002$, $\beta=(0.9, 0.95)$. We trained for 4578 steps. The batch size was $8\times 65536 / L_\text{tr}$, with more gradient accumulation for shorter training lengths. We scheduled the learning rate with a linear schedule for all parameters, with no warmup, constant learning rate for 3270 steps, and linear cooldown for the remaining 1308 steps. We vary the training context length when pretraining with this model, from 4096(4k), to 16384(16k), and 65536(64k). We validate at 4k, 16k, and 64k sequence lengths every 125 steps.

\textbf{304M parameter model. } For the larger model, we used the same settings as the smaller model, but with the following changes. The larger model had 16 layers, 1024 hidden dimension, and 8 heads. We trained the model for 10900 steps, processing approximately 10B tokens. We trained for 2 accumulation steps on 4 GPUs, with 28 sequences per batch. We scaled learning rates following $\mu$Param~\citep{wortsman2023small}, which involves multiplying learning rates of the linear layers by $768/1024$, to adjust for changing the model width. Learning rates of other layers (embedding, unembedding, and LogN scalings if active) were not changed. We used a cosine learning rate schedule, with no warmup, and a minimum learning rate of 0 (at the end of training).  All runs on the 304M parameter models were with 4096 training sequence length. We validated at 4k, 16k, 64k sequence lengths every 250 steps.

\textbf{RoPE hyperparameters. } We used a base $\theta$ of $10,000$ for RoPE, and an effective base of $1024$ for the angular frequencies in $p$-RoPE. For RoPE+NTK scaling (note the scaling applies only when the inference sequence length is longer than the training sequence length), we scaled $\theta$ by $(L_\text{inf}/L_\text{tr})^{d/(d-2)}$, where $L_\text{tr}$ is the training sequence length, $L_\text{inf}$ is the inference sequence length, and $d$ is the RoPE head dimension (128 for both models).

\textbf{Dataset. } We used subsets of FineWeb~\citep{penedo2024fineweb} for pretraining from scratch: a 10B token subset for the 162M model, and a 100B token subset for the 304M model (from which $\sim$10B tokens were used for its training). We keep the training data ordering fixed, and we keep the validation set identical across all runs to reduce variance.

\textbf{Seeds. } We repeat the experiment for 3 different seeds when training at 4k on the 162M parameter model. We train with 1 seed at 4k on the 304M parameter model, and at 16k/64k on the 304M parameter model due to compute limitations.

\textbf{Compute. } We trained the smaller (162M) models on single A100 80G GPUs. We trained the 304M models on 4xH100 grace hopper nodes using distributed data parallelism.

\subsection{Needle-in-a-haystack}
The Needle in a Haystack experiments were conducted by fine-tuning the pre-trained 162M parameter models. We fine-tune with the same learning rate as above, but with 100 warmup, 100 constant, and 100 warmdown steps (decaying to $\gamma=0$). We use the same optimizers as in the pretraining phase. We train for 300 steps on sequences of length 4096, and batch size 8 with 8 accumulation steps. We repeated with 3 seeds.

The task is to generate responses of the form ``\texttt{city1=needle1;city2=needle2;city3=needle3}'', where the cities and needles are embedded uniformly at random into samples from C4 in the form ``\texttt{The special magic <city> number is <7\_digit\_number>.}''. We sample several times from the C4 dataset (concatenating samples) and remove tokens until we have the necessary number of tokens --- 4096(4k) for training, and 4096(4k), 16384(16k), 65536(64k) for validation. Only the expected response tokens are included in the loss (the prompt/context tokens are masked). We repeat the task for three different seeds. The validation accuracies presenting in Table~\ref{tab:final_losses} are calculated by the proportion of times that  cities {\it and\/} numbers are output correctly.

\subsection{Resources required to reproduce experiments}\label{sec:time_repro}
To reproduce results, 80G GPUs are required. We used 80G A100s, and 80G H100 grace hopper nodes.
In terms of time taken to execute each experiment type, we give estimates of the resources required:
\begin{itemize}[topsep=0pt, parsep=0pt, itemsep=1.5pt]
\item pretraining 162M parameter model w/o flex attention takes roughly 4/8/22 A100 hours at 4k/16k/64k;
\item pretraining 162M parameter model with flex attention takes roughly 8/16/44 A100 hours at 4k/16k/64k;
\item pretraining 304M parameter model takes roughly 36 4xH100 node hours;
\item fine-tuning on needle-in-a-haystack task takes roughly 3 A100 hours.
\end{itemize}
To obtain our results, each experiment is executed several times. In particular, needle-in-a-haystack with 3 seeds per method, pretraining at 4k with 162M model is 3 seeds per positional encoding method. The remaining experiments are executed once per positional encoding method. This gives a total of $\sim 100$ 4xH100 hours, $\sim 550$ 1xA100 hours. We estimate that very roughly that amount again was spent configuring the experiments correctly, and on preliminary/failed experiments.

\section{The optimal lengthscale, $\tau$, is around 10}\label{app:setting_lengthscale}
By introducing scale-invariant attention, we introduce one extra lengthscale hyperparameter, $\tau$, which represents the size of the `chunks' we attend over. To select $\tau$, we trained at 4k for ${\tau\in\{10^{-2},10^{-1},10^0, 10^1, 10^2\}}$, and compared validation losses (shown in Fig.~\ref{fig:select_tau}). Validation performance is very similar amongst different $\tau$ at the training context length (though $\tau=10$ is strictly best), but as we extend to out of distribution context lengths (64k, 16$\times$ the training context length) the benefit of $\tau=10$ becomes clearer.
\begin{figure}[t]
    \centering
    \includegraphics[width=\linewidth]{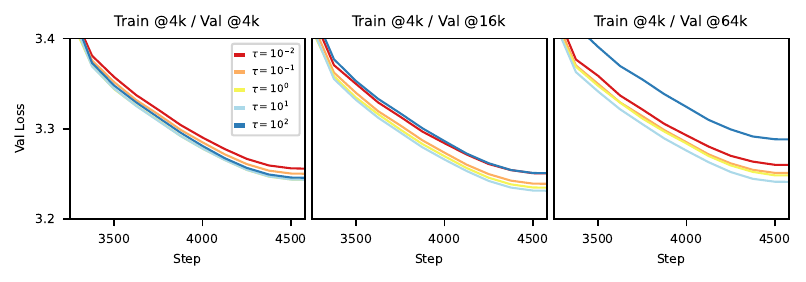}
    \caption{Validation losses at 4k (left), 16k (middle), and 64k (right) context lengths, for a GPT-2-like model trained with scale-invariant attention for varying $\tau$. The models were trained at 4k context length.}
    \label{fig:select_tau}
\end{figure}

\section{Extra Results}
\subsection{Entropy scaling of regular attention}
In the main text, we illustrated in Fig.~\ref{fig:sublog} that scale-invariant attention method under a Gaussian assumption has sub-logarithmic expected entropy. Fig.~\ref{fig:sublog_bounds_std} empirically shows that expected entropy in a range $t$ to $t\Delta$ for standard/unscaled attention instead scales logarithmically with $t$.
\begin{figure}
    \centering
    \includegraphics[width=\linewidth]{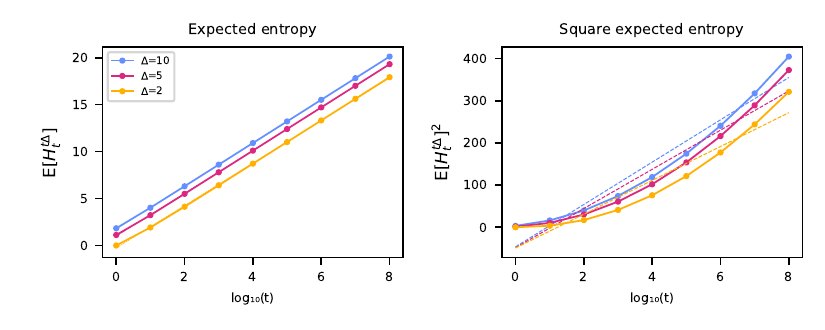}
    \caption{Scaling of expected entropy-in-range for standard attention with an independent Gaussian assumption on the logits. We calculate expected entropies in the range $[t, t\Delta]$ for different $t$ and $\Delta\in\{2, 5, 10\}$, with the lengthscale $\tau$ set to $10$. The dashed lines show the best linear fit of the data. We empirically find that standard Gaussian logits give logarithmic entropy (left panel).}
    \label{fig:sublog_bounds_std}
\end{figure}
\subsection{Pretrain @16k}
For completeness, we include validation losses when training at 16k in Fig.~\ref{fig:pretrain_16k}. We see again that scale-invariant $p$-RoPE outperforms other methods over a range of validation lengths, with the improvements becoming more noticable at the longest validation length of 64k.
\begin{figure}[h]
    \centering
    \includegraphics[width=\linewidth]{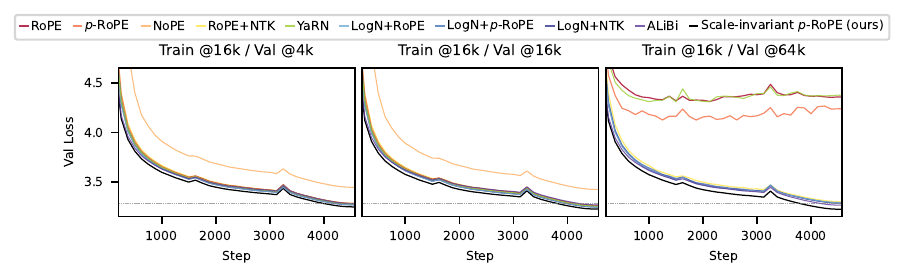}
    \caption{Validation losses at 4k / 16k / 64k context lengths, for a 162M parameter GPT-2-style model trained at 16k. NoPE omitted on the right-most plot to avoid excessive zooming.}
    \label{fig:pretrain_16k}
\end{figure}
\subsection{Alternative scale-invariant attention variants}\label{app:scale_invariant_comparisons}
In Section~\ref{sec:methods} we presented scale-invariant $p$-RoPE as our proposed method. In preliminary experiments however, it was very natural to also consider scale-invariant RoPE and scale-invariant NoPE. We show results when training at 4k in Fig.~\ref{fig:scale_invariant_rope_nope}. Scale-invariant RoPE performs almost as well as scale-invariant $p$-RoPE when evaluating at the training context length, but underperforms more as we move to 16k and 64k. On the other hand, scale-invariant NoPE underperforms scale-invariant $p$-RoPE, yet generalises to long contexts. We hypothesised that scale-invariant RoPE does not long-context generalise due to RoPE's low frequences (i.e. high wavelengths) interfering with the position-dependent scale-invariant transformation, $L_t \mapsto a_t L_t + m_t$.
\begin{figure}[t]
    \centering
    \includegraphics[width=\linewidth]{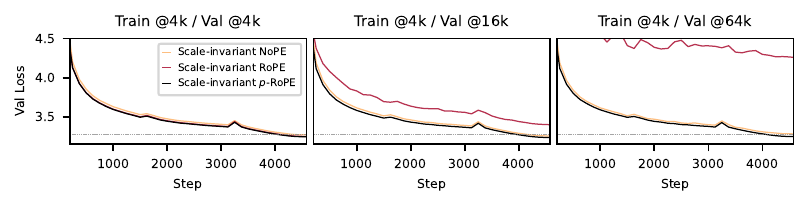}
    \caption{Validation losses at 4k / 16k / 64k of scale-invariant NoPE, RoPE, and $p$-RoPE, for a 162M parameter GPT-2-style model trained at 4k .}
    \label{fig:scale_invariant_rope_nope}
\end{figure}
\subsection{Infini-Attention}
During the review process, we were asked to compare scale-invariant attention to Infini-attention~\citep{munkhdalai2024leave}. Since infini-attention is a method for compressing the KV-cache, it is slightly different to the other methods we compare to (which primarily control entropy), and so we include these results separately. See Table~\ref{tab:infini_attn_results}, which shows that while infini-attention gives long-context generalization, it is not as strong as our method.
\begin{table}[t]
\small
\centering
\caption{Final mean validation losses ($\pm 1$ standard error across 3 seeds) for different methods when training with 4k context length on a 162M parameter GPT-2-style model.}\label{tab:infini_attn_results}
\begin{tabular}{lccc}
\toprule
Method & Val @ 4k & Val @ 16k & Val @ 64k \\
\midrule
RoPE & 3.261 $\pm$ 0.001 & 3.936 $\pm$ 0.010 & 5.260 $\pm$ 0.014 \\
Scale-invariant $p$-RoPE (ours) & \textbf{3.244} $\pm$ 0.001 & \textbf{3.235} $\pm$ 0.001 & \textbf{3.247} $\pm$ 0.001 \\
\midrule
Infini-RoPE & 3.296 $\pm$ 0.003 & 3.302 $\pm$ 0.004 & 3.310 $\pm$ 0.008 \\
Infini-$p$-RoPE & 3.295 $\pm$ 0.003 & 3.303 $\pm$ 0.005 & 3.311 $\pm$ 0.009 \\
\bottomrule
\end{tabular}
\end{table}
\subsection{Larger models}\label{app:larger_models}
We also investigated the ability for scale-invariant attention to generalise to longer contexts in larger models by continual pretraining Llama 2 7B~\citep{touvron2023llama}.
We chose Llama 2 for this experiment because it is one of the few models that have not already mid-trained at a longer context length.

Specifically, we continually pretrained at 4k context length after replacing the default attention mechanism with various other methods, and we looked at validation loss at out-of-distribution lengths (4k/16k/64k). We trained using the ~\cite{torchtune} library, using data from FineWeb~\citep{penedo2024fineweb}, with AdamW and a learning rate of $2\times 10^{-5}$.

We see in Table~\ref{tab:llama_val_loss} that changing the attention mechanism degraded the loss at the training context length, but the performance of scale-invariant $p$-RoPE far exceeds the other methods at 16k and 64k.
\begin{table}[t]
\centering
\caption{Validation performance when fine-tuning Llama-2 7B with different attention methods on $\sim$50M tokens. RoPE and RoPE+NTK (denoted *) were not fine-tuned.}
\label{tab:llama_val_loss}
\small
\begin{tabular}{lccc}
\toprule
Method & Val loss @4k & Val loss @16k & Val loss @64k \\
\midrule
RoPE* & 1.968 & 7.036 & 8.815 \\
$p$-RoPE & 2.029 & 3.504 & 6.800 \\
NoPE & 4.152 & 6.172 & 7.730 \\
RoPE+NTK* & 1.988 & 3.090 & 7.523 \\
YaRN & \textbf{1.957} & 7.004 & 8.763 \\
LogN+RoPE & 1.966 & 6.932 & 8.726 \\
LogN+$p$-RoPE & 2.049 & 2.984 & 6.224 \\
LogN+NTK & 1.966 & 3.063 & 7.413 \\
ALiBi & 2.750 & 2.745 & 2.744 \\
Scale-invariant $p$-RoPE (ours) & 2.163 & \textbf{2.193} & \textbf{2.252} \\
\bottomrule
\end{tabular}
\end{table}
\section{Checking Gaussianity of attention logits}
In our analysis in Section~\ref{sec:logit_characteristics} we assume that the unmodified logits (query-key products), $\bar{L}_t$'s, are standard Gaussian. In this section, we empirically verify that the logits are Gaussian by looking at QQ-plots.

We consider several sizes of model, including 1B and 8B Llama variants~\citep{grattafiori2024llama}, and Gemma 2 27B~\cite{team2024gemma}. We calculate logits with the introduction paragraph of a Wikipedia page~\footnote{\url{https://en.wikipedia.org/wiki/New_England}} as the input. In Figs.~\ref{fig:qq_plot_1b},\ref{fig:qq_plot_1b_post_rope},\ref{fig:qq_plot_8b},\ref{fig:qq_plot_8b_post_rope},\ref{fig:qq_plot_gemma},\ref{fig:qq_plot_gemma_post_rope} we show quantiles of $\{\bar{L}_t\}_{t > 0}$'s (i.e. lower triangular part of the $QK^T$ matrix) for each layer, aggregated over the input and the attention heads in each layer. Note that we do not aggregate over the `beginning of sequence' token, as it is an outlier attention sink~\citep{gu2024attention}.
\begin{figure}
    \centering
    \includegraphics[width=\linewidth]{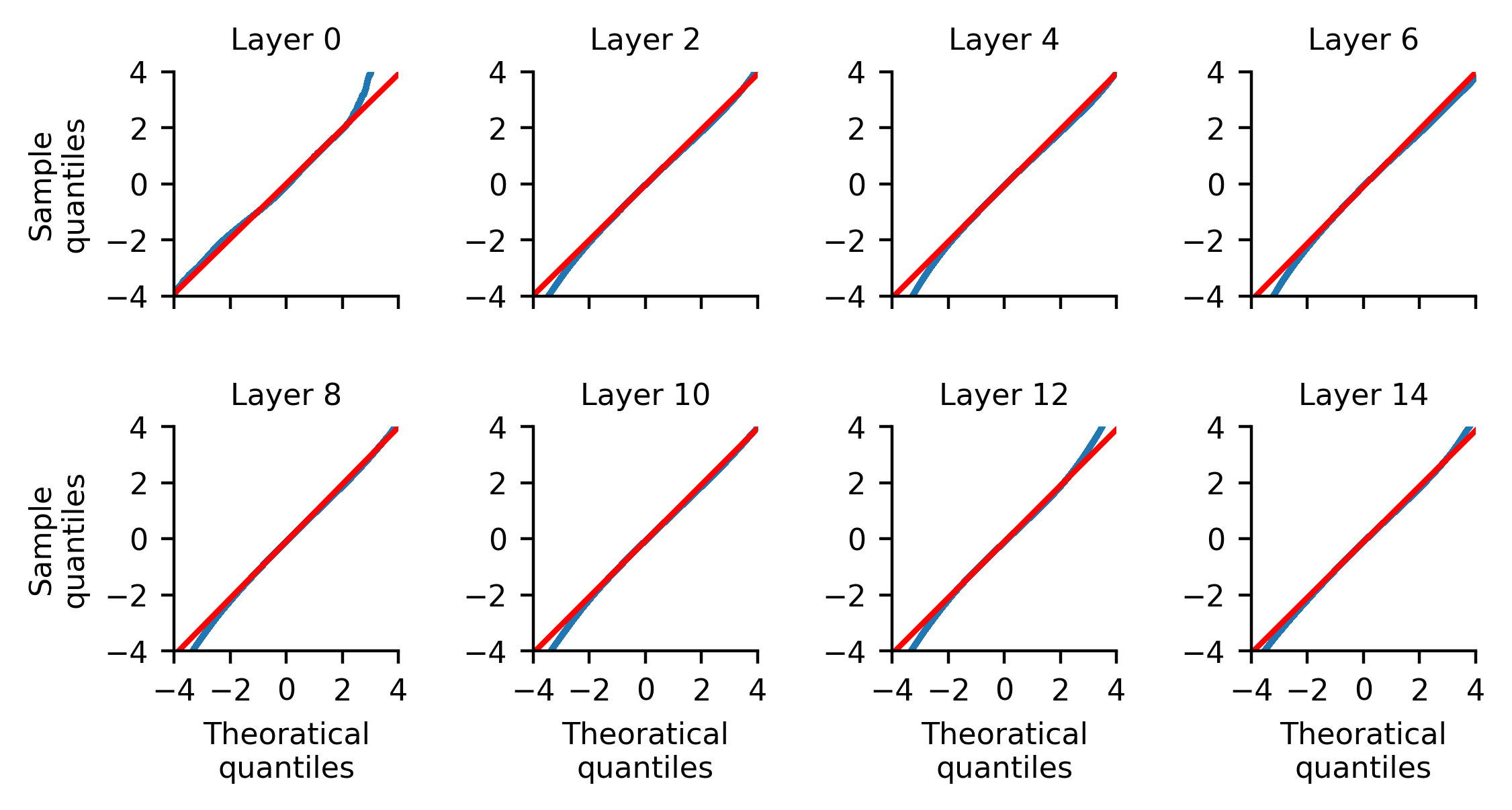}
    \caption{Quantile-Quantile (QQ) plots of attention logits (without RoPE applied) in a Llama-1B model (blue line), with theoretical quantiles of a Gaussian shown by the red line.}
    \label{fig:qq_plot_1b}
\end{figure}

\begin{figure}
    \centering
    \includegraphics[width=\linewidth]{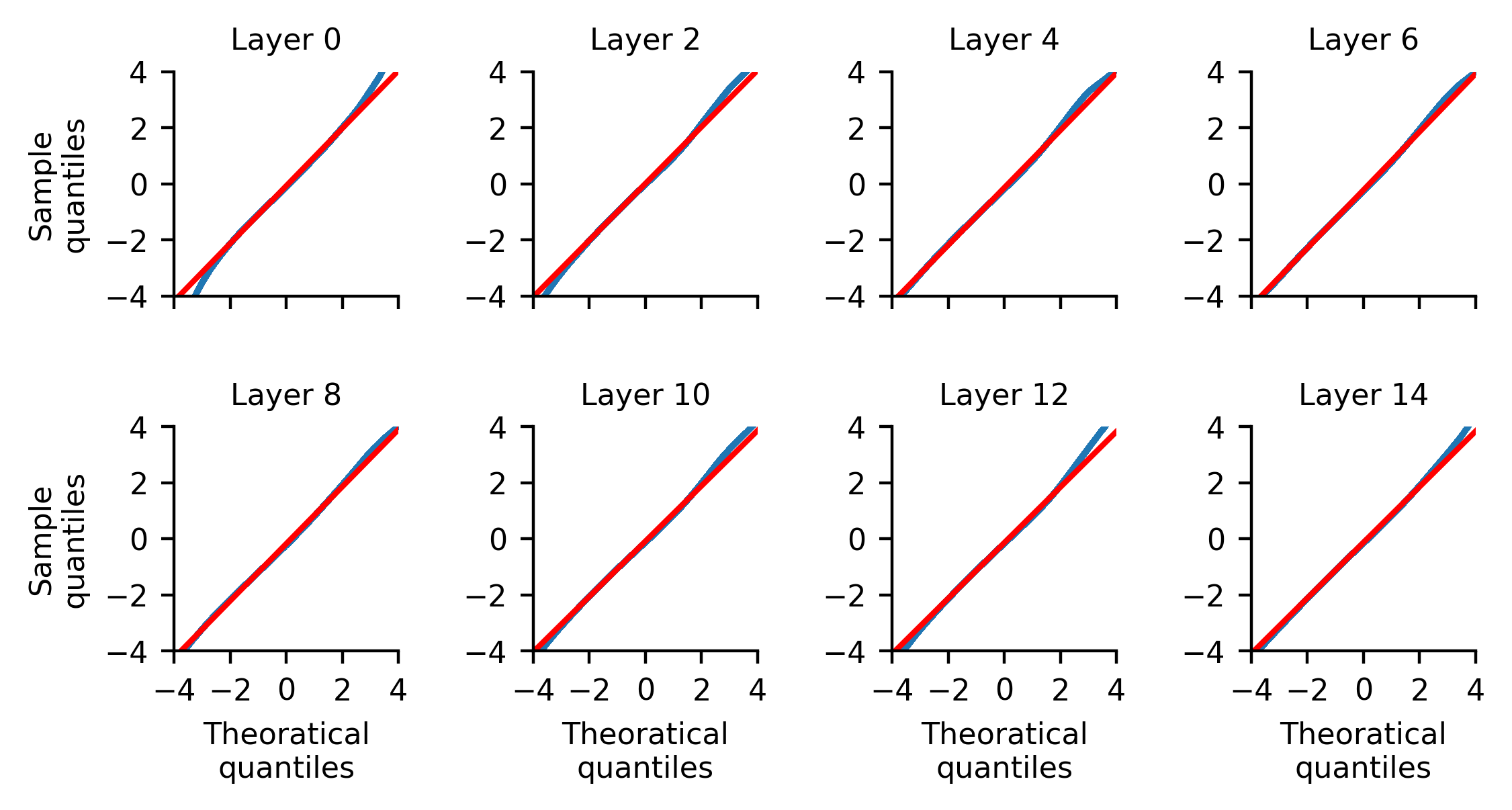}
    \caption{Quantile-Quantile (QQ) plots of attention logits (with RoPE applied) in a Llama-1B model (blue line), with theoretical quantiles of a Gaussian shown by the red line.}
    \label{fig:qq_plot_1b_post_rope}
\end{figure}
\begin{figure}
    \centering
    \includegraphics[width=\linewidth]{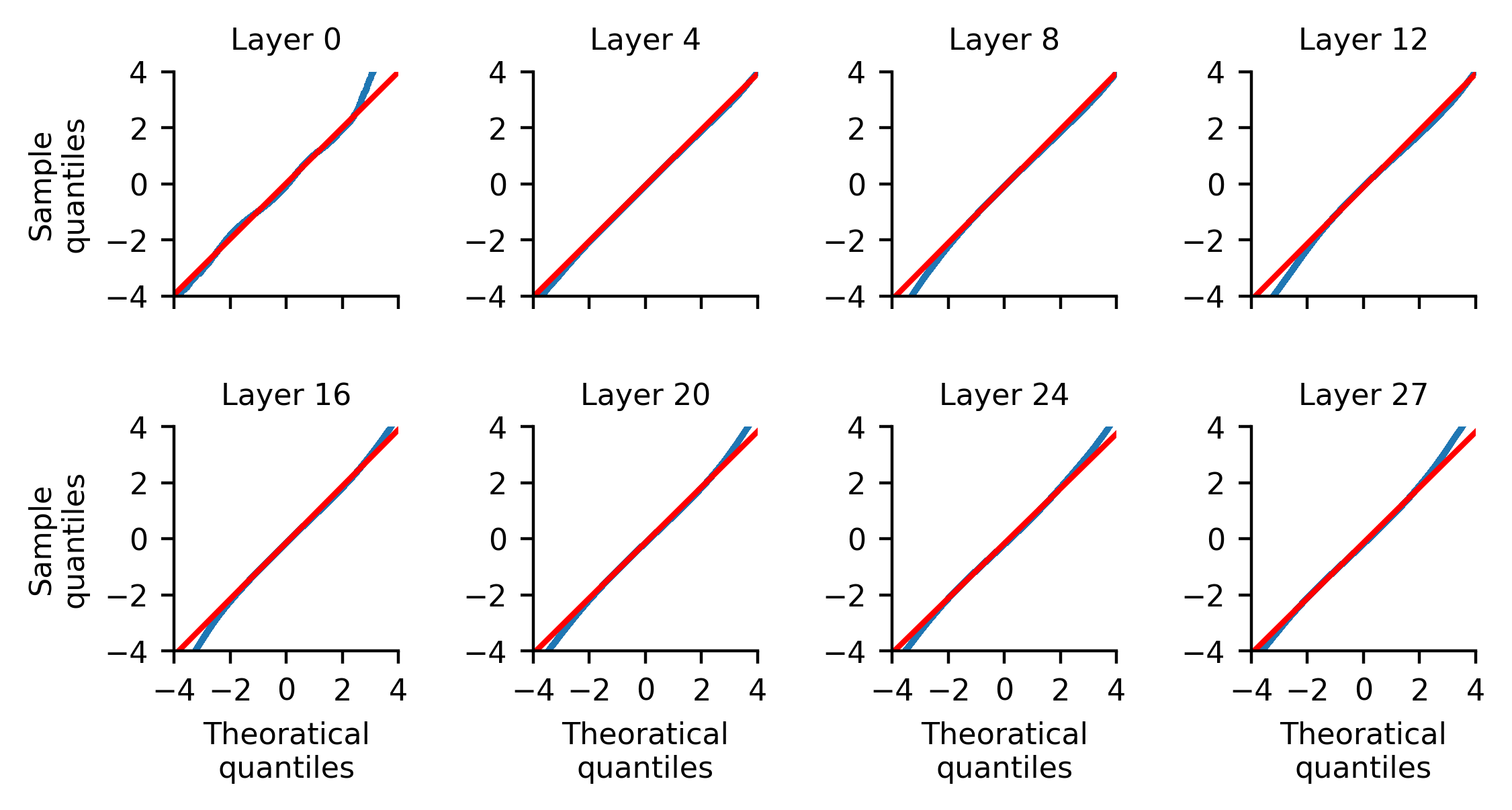}
    \caption{Quantile-Quantile (QQ) plots of attention logits (without RoPE applied) in a Llama-8B model (blue line), with theoretical quantiles of a Gaussian shown by the red line.}
    \label{fig:qq_plot_8b}
\end{figure}

\begin{figure}
    \centering
    \includegraphics[width=\linewidth]{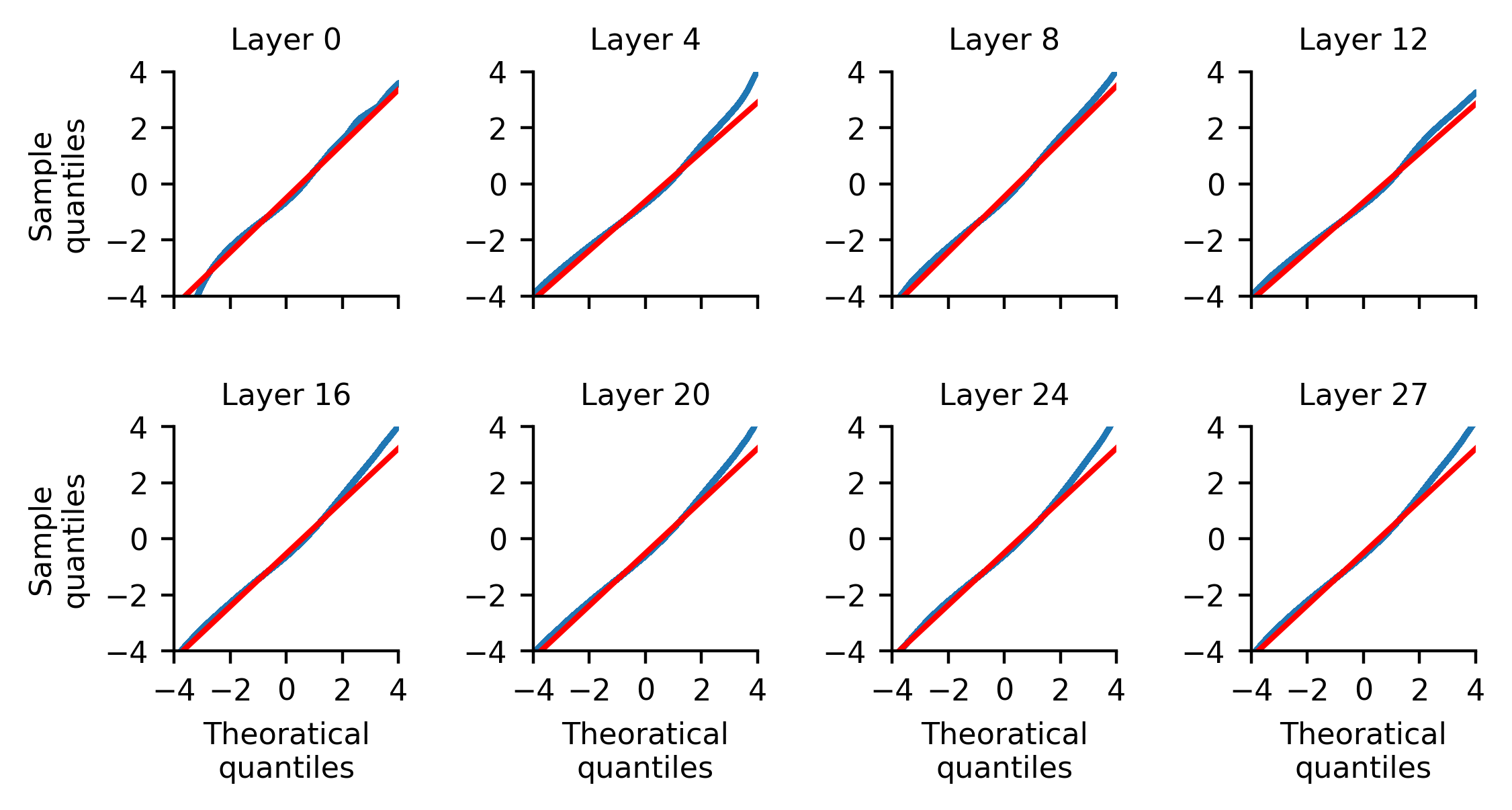}
    \caption{Quantile-Quantile (QQ) plots of attention logits (with RoPE applied) in a Llama-8B model (blue line), with theoretical quantiles of a Gaussian shown by the red line.}
    \label{fig:qq_plot_8b_post_rope}
\end{figure}
\begin{figure}
    \centering
    \includegraphics[width=\linewidth]{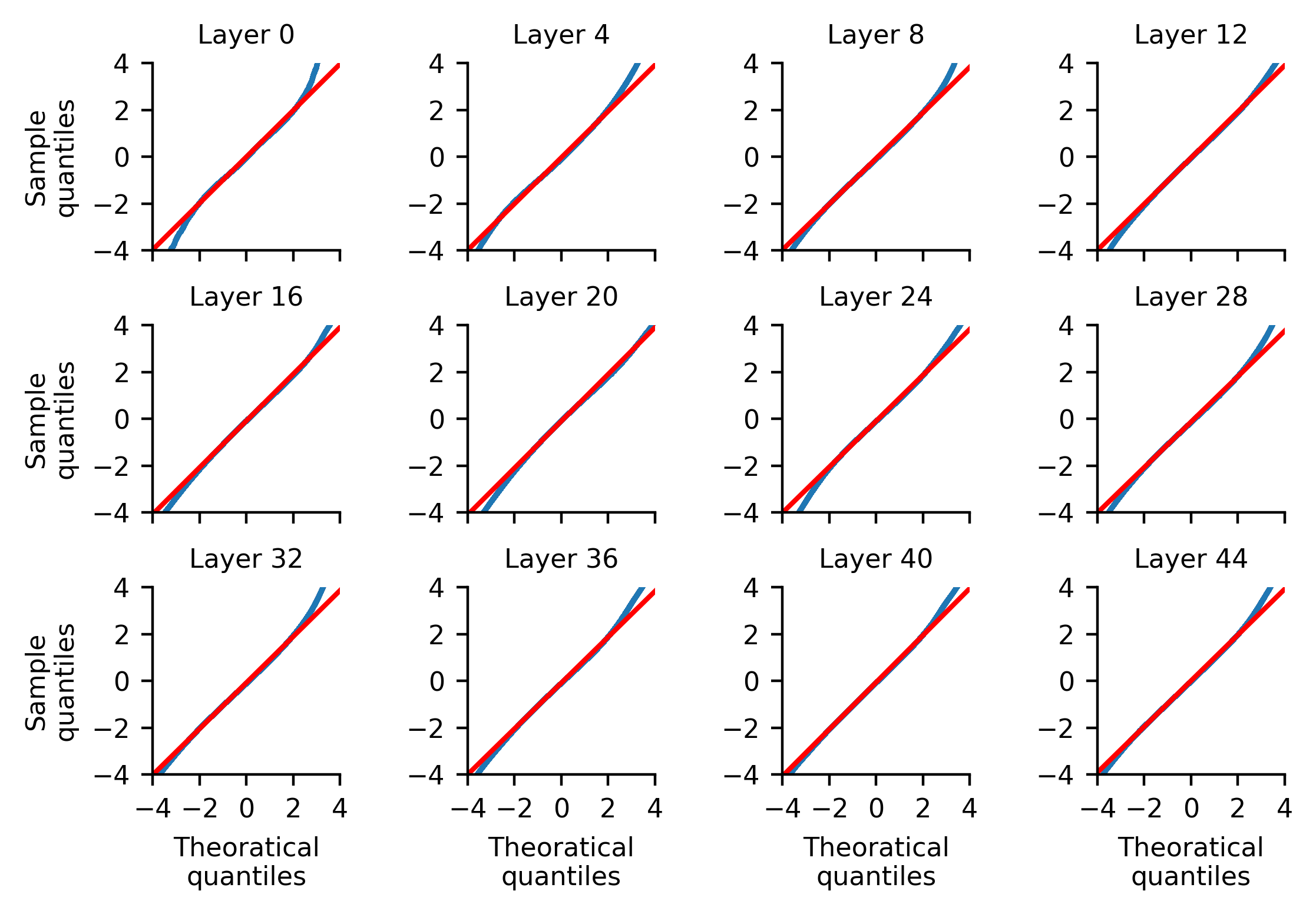}
    \caption{Quantile-Quantile (QQ) plots of attention logits (without RoPE applied) in a Gemma 2 27B model (blue line), with theoretical quantiles of a Gaussian shown by the red line.}
    \label{fig:qq_plot_gemma}
\end{figure}
\begin{figure}
    \centering
    \includegraphics[width=\linewidth]{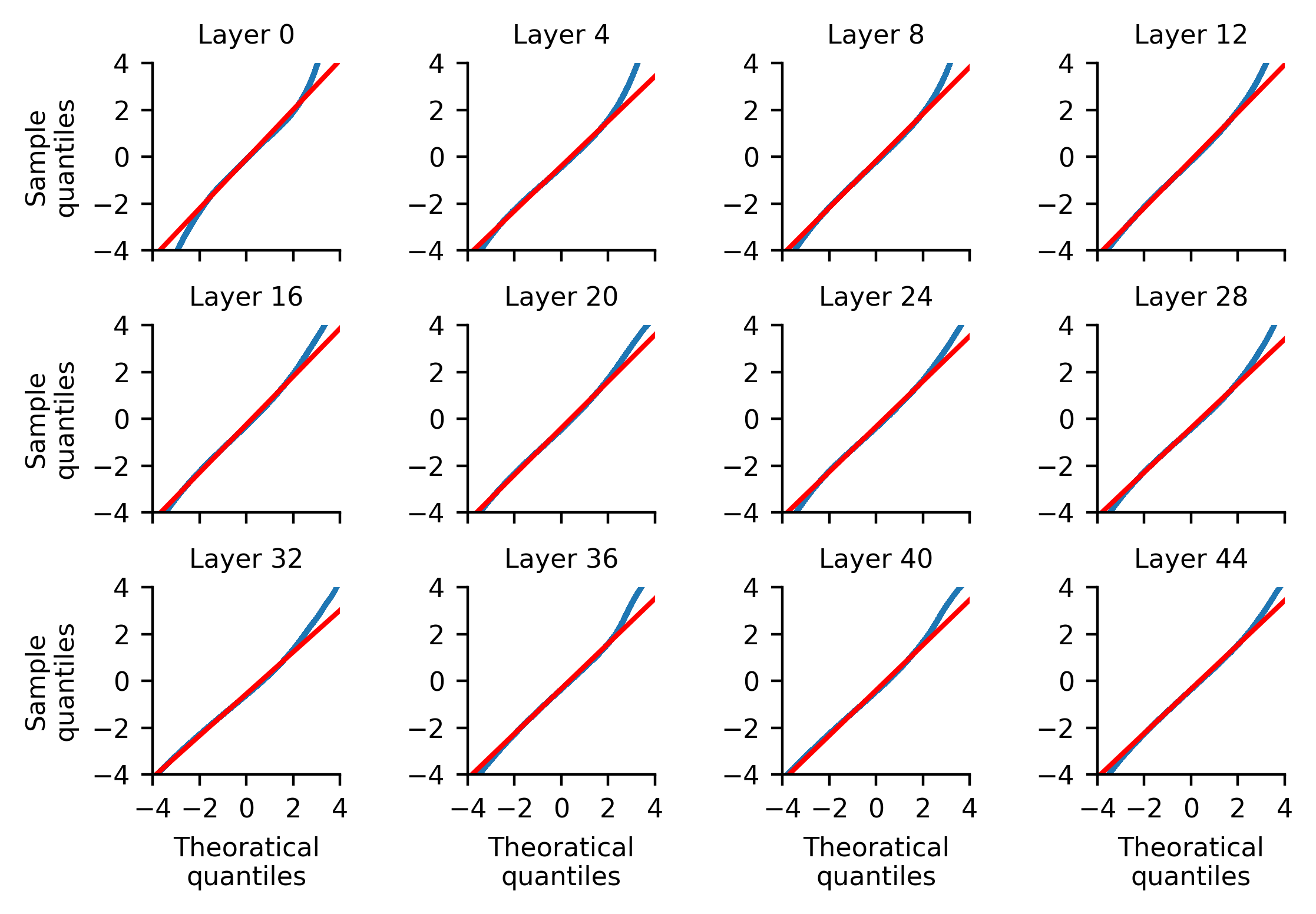}
    \caption{Quantile-Quantile (QQ) plots of attention logits (with RoPE applied) in a Gemma 2 27B model (blue line), with theoretical quantiles of a Gaussian shown by the red line.}
    \label{fig:qq_plot_gemma_post_rope}
\end{figure}

\FloatBarrier
\section{Licenses}
\begin{itemize}[topsep=0pt, parsep=0pt, itemsep=1.5pt]
\item This project uses a modified version of \texttt{modded-nanogpt}, \url{https://github.com/KellerJordan/modded-nanogpt} which is MIT licensed.
\item This project uses a 10B subset of the fineweb dataset, \url{https://huggingface.co/datasets/kjj0/fineweb10B-gpt2}, which is MIT licensed.
\item This project uses a 100B subset of the fineweb dataset, \url{https://huggingface.co/datasets/kjj0/fineweb100B-gpt2}, which is MIT licensed.
\item This project uses the C4 dataset, \url{https://huggingface.co/datasets/kjj0/fineweb100B-gpt2}, which licensed under the Open Civic Data Attribution License (OCD-BY).
\end{itemize}

%% file: refs.bib
@article{van1996modelling,
  title={Modelling the power spectra of natural images: statistics and information},
  author={Van der Schaaf, van A and van Hateren, JH van},
  journal={Vision research},
  volume={36},
  number={17},
  pages={2759--2770},
  year={1996},
  publisher={Elsevier}
}

@article{dong2024flex,
  title={Flex Attention: A Programming Model for Generating Optimized Attention Kernels},
  author={Dong, Juechu and Feng, Boyuan and Guessous, Driss and Liang, Yanbo and He, Horace},
  journal={arXiv preprint arXiv:2412.05496},
  year={2024}
}

@article{barbero2024transformers,
  title={Transformers need glasses! information over-squashing in language tasks},
  author={Barbero, Federico and Banino, Andrea and Kapturowski, Steven and Kumaran, Dharshan and Madeira Ara{\'u}jo, Jo{\~a}o and Vitvitskyi, Oleksandr and Pascanu, Razvan and Veli{\v{c}}kovi{\'c}, Petar},
  journal={Advances in Neural Information Processing Systems},
  volume={37},
  pages={98111--98142},
  year={2024}
}

@article{press2021train,
  title={Train short, test long: Attention with linear biases enables input length extrapolation},
  author={Press, Ofir and Smith, Noah A and Lewis, Mike},
  journal={arXiv preprint arXiv:2108.12409},
  year={2021}
}

@article{su2024roformer,
  title={Roformer: Enhanced transformer with rotary position embedding},
  author={Su, Jianlin and Ahmed, Murtadha and Lu, Yu and Pan, Shengfeng and Bo, Wen and Liu, Yunfeng},
  journal={Neurocomputing},
  volume={568},
  pages={127063},
  year={2024},
  publisher={Elsevier}
}

@article{chen2023extending,
  title={Extending context window of large language models via positional interpolation},
  author={Chen, Shouyuan and Wong, Sherman and Chen, Liangjian and Tian, Yuandong},
  journal={arXiv preprint arXiv:2306.15595},
  year={2023}
}

@article{peng2023yarn,
  title={Yarn: Efficient context window extension of large language models},
  author={Peng, Bowen and Quesnelle, Jeffrey and Fan, Honglu and Shippole, Enrico},
  journal={arXiv preprint arXiv:2309.00071},
  year={2023}
}

@misc{bloc97_2023,
  author = {bloc97},
  title = {NTK-Aware Scaled RoPE allows LLaMA models to have extended (8k+) context size without any fine-tuning and minimal perplexity degradation},
  year = {2023},
  howpublished = {Reddit},
  url = {https://www.reddit.com/r/LocalLLaMA/comments/14lz7j5/ntkaware_scaled_rope_allows_llama_models_to_have/}
}

@article{nakanishi2025scalable,
  title={Scalable-Softmax Is Superior for Attention},
  author={Nakanishi, Ken M},
  journal={arXiv preprint arXiv:2501.19399},
  year={2025}
}

@article{chiang2022overcoming,
  title={Overcoming a theoretical limitation of self-attention},
  author={Chiang, David and Cholak, Peter},
  journal={arXiv preprint arXiv:2202.12172},
  year={2022}
}

@article{li2025information,
  title={Information Entropy Invariance: Enhancing Length Extrapolation in Attention Mechanisms},
  author={Li, Kewei and Kong, Yanwen and Xu, Yiping and Su, Jianlin and Huang, Lan and Zhang, Ruochi and Zhou, Fengfeng},
  journal={arXiv preprint arXiv:2501.08570},
  year={2025}
}

@article{xiong2023effective,
  title={Effective long-context scaling of foundation models},
  author={Xiong, Wenhan and Liu, Jingyu and Molybog, Igor and Zhang, Hejia and Bhargava, Prajjwal and Hou, Rui and Martin, Louis and Rungta, Rashi and Sankararaman, Karthik Abinav and Oguz, Barlas and others},
  journal={arXiv preprint arXiv:2309.16039},
  year={2023}
}

@article{grattafiori2024llama,
  title={The llama 3 herd of models},
  author={Grattafiori, Aaron and Dubey, Abhimanyu and Jauhri, Abhinav and Pandey, Abhinav and Kadian, Abhishek and Al-Dahle, Ahmad and Letman, Aiesha and Mathur, Akhil and Schelten, Alan and Vaughan, Alex and others},
  journal={arXiv preprint arXiv:2407.21783},
  year={2024}
}

@article{han2023lm,
  title={Lm-infinite: Zero-shot extreme length generalization for large language models},
  author={Han, Chi and Wang, Qifan and Peng, Hao and Xiong, Wenhan and Chen, Yu and Ji, Heng and Wang, Sinong},
  journal={arXiv preprint arXiv:2308.16137},
  year={2023}
}

@article{fountas2024human,
  title={Human-like episodic memory for infinite context llms},
  author={Fountas, Zafeirios and Benfeghoul, Martin A and Oomerjee, Adnan and Christopoulou, Fenia and Lampouras, Gerasimos and Bou-Ammar, Haitham and Wang, Jun},
  journal={arXiv preprint arXiv:2407.09450},
  year={2024}
}

@article{xiao2024infllm,
  title={Infllm: Unveiling the intrinsic capacity of llms for understanding extremely long sequences with training-free memory},
  author={Xiao, Chaojun and Zhang, Pengle and Han, Xu and Xiao, Guangxuan and Lin, Yankai and Zhang, Zhengyan and Liu, Zhiyuan and Han, Song and Sun, Maosong},
  journal={arXiv e-prints},
  pages={arXiv--2402},
  year={2024}
}

@article{zhu2023pose,
  title={Pose: Efficient context window extension of llms via positional skip-wise training},
  author={Zhu, Dawei and Yang, Nan and Wang, Liang and Song, Yifan and Wu, Wenhao and Wei, Furu and Li, Sujian},
  journal={arXiv preprint arXiv:2309.10400},
  year={2023}
}

@article{wang2024resonance,
  title={Resonance rope: Improving context length generalization of large language models},
  author={Wang, Suyuchen and Kobyzev, Ivan and Lu, Peng and Rezagholizadeh, Mehdi and Liu, Bang},
  journal={arXiv preprint arXiv:2403.00071},
  year={2024}
}

@article{bai2024longalign,
  title={Longalign: A recipe for long context alignment of large language models},
  author={Bai, Yushi and Lv, Xin and Zhang, Jiajie and He, Yuze and Qi, Ji and Hou, Lei and Tang, Jie and Dong, Yuxiao and Li, Juanzi},
  journal={arXiv preprint arXiv:2401.18058},
  year={2024}
}

@article{ye2024mplug,
  title={mplug-owl3: Towards long image-sequence understanding in multi-modal large language models},
  author={Ye, Jiabo and Xu, Haiyang and Liu, Haowei and Hu, Anwen and Yan, Ming and Qian, Qi and Zhang, Ji and Huang, Fei and Zhou, Jingren},
  journal={arXiv preprint arXiv:2408.04840},
  year={2024}
}

@article{jin2024llm,
  title={Llm maybe longlm: Self-extend llm context window without tuning},
  author={Jin, Hongye and Han, Xiaotian and Yang, Jingfeng and Jiang, Zhimeng and Liu, Zirui and Chang, Chia-Yuan and Chen, Huiyuan and Hu, Xia},
  journal={arXiv preprint arXiv:2401.01325},
  year={2024}
}

@article{beltagy2020longformer,
  title={Longformer: The long-document transformer},
  author={Beltagy, Iz and Peters, Matthew E and Cohan, Arman},
  journal={arXiv preprint arXiv:2004.05150},
  year={2020}
}

@article{ding2023longnet,
  title={Longnet: Scaling transformers to 1,000,000,000 tokens},
  author={Ding, Jiayu and Ma, Shuming and Dong, Li and Zhang, Xingxing and Huang, Shaohan and Wang, Wenhui and Zheng, Nanning and Wei, Furu},
  journal={arXiv preprint arXiv:2307.02486},
  year={2023}
}

@article{munkhdalai2024leave,
  title={Leave no context behind: Efficient infinite context transformers with infini-attention},
  author={Munkhdalai, Tsendsuren and Faruqui, Manaal and Gopal, Siddharth},
  journal={arXiv preprint arXiv:2404.07143},
  volume={101},
  year={2024}
}

@inproceedings{bulatov2024beyond,
  title={Beyond attention: breaking the limits of transformer context length with recurrent memory},
  author={Bulatov, Aydar and Kuratov, Yuri and Kapushev, Yermek and Burtsev, Mikhail},
  booktitle={Proceedings of the AAAI Conference on Artificial Intelligence},
  year={2024}
}

@article{liu2023ring,
  title={Ring attention with blockwise transformers for near-infinite context},
  author={Liu, Hao and Zaharia, Matei and Abbeel, Pieter},
  journal={arXiv preprint arXiv:2310.01889},
  year={2023}
}

@article{barbero2024round,
  title={Round and Round We Go! What makes Rotary Positional Encodings useful?},
  author={Barbero, Federico and Vitvitskyi, Alex and Perivolaropoulos, Christos and Pascanu, Razvan and Veli{\v{c}}kovi{\'c}, Petar},
  journal={arXiv preprint arXiv:2410.06205},
  year={2024}
}

@article{hu2024longrecipe,
  title={LongRecipe: Recipe for Efficient Long Context Generalization in Large Language Models},
  author={Hu, Zhiyuan and Liu, Yuliang and Zhao, Jinman and Wang, Suyuchen and Wang, Yan and Shen, Wei and Gu, Qing and Luu, Anh Tuan and Ng, See-Kiong and Jiang, Zhiwei and others},
  journal={arXiv preprint arXiv:2409.00509},
  year={2024}
}

@misc{modded_nanogpt_2024,
  author       = {Keller Jordan and Jeremy Bernstein and Brendan Rappazzo and
                  @fernbear.bsky.social and Boza Vlado and You Jiacheng and
                  Franz Cesista and Braden Koszarsky and @Grad62304977},
  title        = {modded-nanogpt: Speedrunning the NanoGPT baseline},
  year         = {2024},
  url          = {https://github.com/KellerJordan/modded-nanogpt}
}

@misc{karpathy_llmc_481,
  author = {Karpathy, Andrej},
  title = {Discussion \#481},
  year = {2024},
  howpublished = {\url{https://github.com/karpathy/llm.c/discussions/481}}
}

@article{penedo2024fineweb,
  title={The fineweb datasets: Decanting the web for the finest text data at scale},
  author={Penedo, Guilherme and Kydl{\'\i}{\v{c}}ek, Hynek and Lozhkov, Anton and Mitchell, Margaret and Raffel, Colin A and Von Werra, Leandro and Wolf, Thomas and others},
  journal={Advances in Neural Information Processing Systems},
  volume={37},
  pages={30811--30849},
  year={2024}
}

@article{roberts2019exploring,
  title={Exploring the limits of transfer learning with a unified text-to-text transformer},
  author={Roberts, Adam and Raffel, Colin and Lee, Katherine and Matena, Michael and Shazeer, Noam and Liu, Peter J and Narang, Sharan and Li, Wei and Zhou, Yanqi},
  journal={Google Research},
  year={2019}
}

@article{team2025gemma,
  title={Gemma 3 Technical Report},
  author={{Gemma Team}},
  journal={arXiv preprint arXiv:2503.19786},
  year={2025}
}

@misc{kexuefm-8823,
    title={Attention's Scale Operation from the Perspective of Entropy Invariance},
    author={Su Jianlin},
    year={2021},
    url={{https://www.kexue.fm/archives/8823}},
}

@article{bai2023qwen,
  title={Qwen technical report},
  author={Bai, Jinze and Bai, Shuai and Chu, Yunfei and Cui, Zeyu and Dang, Kai and Deng, Xiaodong and Fan, Yang and Ge, Wenbin and Han, Yu and Huang, Fei and others},
  journal={arXiv preprint arXiv:2309.16609},
  year={2023}
}

@article{liu2024reattention,
  title={ReAttention: Training-Free Infinite Context with Finite Attention Scope},
  author={Liu, Xiaoran and Li, Ruixiao and Guo, Qipeng and Liu, Zhigeng and Song, Yuerong and Lv, Kai and Yan, Hang and Li, Linlin and Liu, Qun and Qiu, Xipeng},
  journal={arXiv preprint arXiv:2407.15176},
  year={2024}
}

@article{martins2021infty,
  title={{$\infty$-former: Infinite Memory Transformer}},
  author={Martins, Pedro Henrique and Marinho, Zita and Martins, Andr{\'e} FT},
  journal={arXiv preprint arXiv:2109.00301},
  year={2021}
}

@article{chen2025edgeinfinite,
  title={EdgeInfinite: A Memory-Efficient Infinite-Context Transformer for Edge Devices},
  author={Chen, Jiyu and Peng, Shuang and Luo, Daxiong and Yang, Fan and Wu, Renshou and Li, Fangyuan and Chen, Xiaoxin},
  journal={arXiv preprint arXiv:2503.22196},
  year={2025}
}

@article{gao2024train,
  title={How to train long-context language models (effectively)},
  author={Gao, Tianyu and Wettig, Alexander and Yen, Howard and Chen, Danqi},
  journal={arXiv preprint arXiv:2410.02660},
  year={2024}
}

@article{lieber2024jamba,
  title={Jamba: A hybrid transformer-mamba language model},
  author={Lieber, Opher and Lenz, Barak and Bata, Hofit and Cohen, Gal and Osin, Jhonathan and Dalmedigos, Itay and Safahi, Erez and Meirom, Shaked and Belinkov, Yonatan and Shalev-Shwartz, Shai and others},
  journal={arXiv preprint arXiv:2403.19887},
  year={2024}
}

@article{yang2024qwen2,
  title={Qwen2. 5 technical report},
  author={Yang, An and Yang, Baosong and Zhang, Beichen and Hui, Binyuan and Zheng, Bo and Yu, Bowen and Li, Chengyuan and Liu, Dayiheng and Huang, Fei and Wei, Haoran and others},
  journal={arXiv preprint arXiv:2412.15115},
  year={2024}
}

@article{liu2024deepseek,
  title={Deepseek-v3 technical report},
  author={Liu, Aixin and Feng, Bei and Xue, Bing and Wang, Bingxuan and Wu, Bochao and Lu, Chengda and Zhao, Chenggang and Deng, Chengqi and Zhang, Chenyu and Ruan, Chong and others},
  journal={arXiv preprint arXiv:2412.19437},
  year={2024}
}

@article{cohere2025command,
  title={Command A: An Enterprise-Ready Large Language Model},
  author={{Cohere Team}},
  journal={arXiv preprint arXiv:2504.00698},
  year={2025}
}

@misc{llama4team2025,
  title={The Llama 4 herd: The beginning of a new era of natively multimodal AI innovation},
  author={{Llama 4 Team}},
  year={2025},
  howpublished={\url{https://ai.meta.com/blog/llama-4-multimodal-intelligence/}},
  note={Accessed: 2025-04-08}
}

@article{kazemnejad2023impact,
  title={The impact of positional encoding on length generalization in transformers},
  author={Kazemnejad, Amirhossein and Padhi, Inkit and Natesan Ramamurthy, Karthikeyan and Das, Payel and Reddy, Siva},
  journal={Advances in Neural Information Processing Systems},
  volume={36},
  pages={24892--24928},
  year={2023}
}

@article{liu2023scaling,
  title={Scaling laws of rope-based extrapolation},
  author={Liu, Xiaoran and Yan, Hang and Zhang, Shuo and An, Chenxin and Qiu, Xipeng and Lin, Dahua},
  journal={arXiv preprint arXiv:2310.05209},
  year={2023}
}

@misc{kamradt2023needle,
  author       = {Kamradt, G.},
  title        = {Needle in a Haystack - Pressure Testing LLMs},
  year         = {2023},
  howpublished = {\url{https://github.com/gkamradt/LLMTest_NeedleInAHaystack}},
}

@article{gu2024attention,
  title={When Attention Sink Emerges in Language Models: An Empirical View},
  author={Gu, Xiangming and Pang, Tianyu and Du, Chao and Liu, Qian and Zhang, Fengzhuo and Du, Cunxiao and Wang, Ye and Lin, Min},
  journal={arXiv preprint arXiv:2410.10781},
  year={2024}
}

@misc{jordan2024muon,
  author       = {Keller Jordan and Yuchen Jin and Vlado Boza and You Jiacheng and
                  Franz Cesista and Laker Newhouse and Jeremy Bernstein},
  title        = {Muon: An optimizer for hidden layers in neural networks},
  year         = {2024},
  url          = {https://kellerjordan.github.io/posts/muon/}
}

@article{wortsman2023small,
  title={Small-scale proxies for large-scale transformer training instabilities},
  author={Wortsman, Mitchell and Liu, Peter J and Xiao, Lechao and Everett, Katie and Alemi, Alex and Adlam, Ben and Co-Reyes, John D and Gur, Izzeddin and Kumar, Abhishek and Novak, Roman and others},
  journal={arXiv preprint arXiv:2309.14322},
  year={2023}
}

@inproceedings{wolf2020transformers,
  title={Transformers: State-of-the-art natural language processing},
  author={Wolf, Thomas and Debut, Lysandre and Sanh, Victor and Chaumond, Julien and Delangue, Clement and Moi, Anthony and Cistac, Pierric and Rault, Tim and Louf, R{\'e}mi and Funtowicz, Morgan and others},
  booktitle={Proceedings of the 2020 conference on empirical methods in natural language processing: system demonstrations},
  pages={38--45},
  year={2020}
}

@article{radford2019language,
  title={Language models are unsupervised multitask learners},
  author={Radford, Alec and Wu, Jeffrey and Child, Rewon and Luan, David and Amodei, Dario and Sutskever, Ilya and others},
  journal={OpenAI blog},
  volume={1},
  number={8},
  pages={9},
  year={2019}
}

@article{touvron2023llama,
  title={Llama 2: Open foundation and fine-tuned chat models},
  author={Touvron, Hugo and Martin, Louis and Stone, Kevin and Albert, Peter and Almahairi, Amjad and Babaei, Yasmine and Bashlykov, Nikolay and Batra, Soumya and Bhargava, Prajjwal and Bhosale, Shruti and others},
  journal={arXiv preprint arXiv:2307.09288},
  year={2023}
}

@software{torchtune,
  title        = {Torchtune: PyTorch's Finetuning Library},
  author       = {{Torchtune}},
  year         = {2024},
  month        = apr,
  url          = {https://github.com/pytorch/torchtune},
  license      = {BSD-3-Clause},
}

@article{team2024gemma,
  title={Gemma 2: Improving open language models at a practical size},
  author={Team, Gemma and Riviere, Morgane and Pathak, Shreya and Sessa, Pier Giuseppe and Hardin, Cassidy and Bhupatiraju, Surya and Hussenot, L{\'e}onard and Mesnard, Thomas and Shahriari, Bobak and Ram{\'e}, Alexandre and others},
  journal={arXiv preprint arXiv:2408.00118},
  year={2024}
}
